\documentclass{article}

\usepackage{microtype}
\usepackage{graphicx}
\usepackage{subcaption}
\usepackage{booktabs}
\usepackage{enumitem}
\setlist[itemize]{leftmargin=*, labelsep=.3em}

\PassOptionsToPackage{numbers, sort&compress}{natbib}

\usepackage[preprint]{arxiv/neurips_2026}
\usepackage{amsmath}
\usepackage{amssymb}
\usepackage{mathtools}
\usepackage{amsthm}

\usepackage[dvipsnames,svgnames,x11names]{xcolor}
\definecolor{lightblue}{HTML}{2970CC}
\definecolor{lightpurple}{HTML}{673147}
\definecolor{ForestGreen}{HTML}{FF5733}
\definecolor{myred}{HTML}{AA4A44}
\usepackage[colorlinks=true,linkcolor=lightblue,urlcolor=lightblue,citecolor=lightblue,breaklinks]{hyperref}
\usepackage[capitalize,noabbrev]{cleveref}

\theoremstyle{plain}
\newtheorem{theorem}{Theorem}[section]
\newtheorem{proposition}[theorem]{Proposition}
\newtheorem{lemma}[theorem]{Lemma}

\theoremstyle{definition}
\newtheorem{definition}[theorem]{Definition}

\theoremstyle{remark}
\newtheorem{remark}[theorem]{Remark}
\usepackage[ruled,noend]{algorithm2e}

\newcommand{\E}{\mathbb{E}}
\newcommand{\R}{\mathbb{R}}
\newcommand{\OmegaR}{\Omega_R}

\renewcommand{\ge}{\geqslant}
\renewcommand{\geq}{\geqslant}

\usepackage{tcolorbox}
\tcbuselibrary{skins,breakable}

\definecolor{keyblue}{RGB}{229,238,249}
\definecolor{keyblueborder}{RGB}{41,112,204}

\newtcolorbox{keyresult}{
	enhanced,
	colback=keyblue,
	colframe=keyblueborder,
	boxrule=1pt,
	arc=3pt,
	left=8pt,
	right=8pt,
	top=6pt,
	bottom=6pt,
	breakable
}

\newtcolorbox{highlight}{
	enhanced,
	colback=keyblue,
	colframe=keyblue,
	boxrule=0pt,
	arc=2pt,
	left=8pt,
	right=8pt,
	top=6pt,
	bottom=6pt,
	breakable
}

\usepackage{mdframed}
\newmdenv[
  linecolor=pink!60!purple,
  backgroundcolor=pink!30,
  linewidth=0pt,
  roundcorner=2pt,
  skipabove=3pt,
  skipbelow=1pt,
  innerleftmargin=3pt,
  innerrightmargin=3pt,
  innertopmargin=3pt,
  innerbottommargin=3pt
]{pinkbox}

\newmdenv[
  linecolor=blue!60!cyan,
  backgroundcolor=blue!10,
  linewidth=0pt,
  roundcorner=2pt,
  skipabove=3pt,
  skipbelow=1pt,
  innerleftmargin=3pt,
  innerrightmargin=3pt,
  innertopmargin=3pt,
  innerbottommargin=3pt
]{bluebox}

\tcbset{colback=blue!5!white, colframe=blue!40!black, boxrule=0.5pt, arc=3pt, left=2pt, right=2pt, top=3pt, bottom=3pt}
\newtcolorbox{propbox}{colback=blue!5!white, colframe=blue!40!black, boxrule=0.5pt, arc=3pt, left=1pt, right=1pt, top=2pt, bottom=2pt}

\usepackage[utf8]{inputenc} 
\usepackage[T1]{fontenc}    
\usepackage{url}            
\usepackage{booktabs}       
\usepackage{amsfonts}       
\usepackage{nicefrac}       
\usepackage{microtype}      
\usepackage{xcolor}         

\title{Reactive Flux Matching:\\[0.3em] \Large{Mechanism Discovery and Adaptive Sampling of Rare Events}}

\author{%
  Rishal Aggarwal \\
  CMU-Pitt Program in Computational Biology\\
  Dept. of Computational \& Systems Biology\\
  University of Pittsburgh\\
  \texttt{rishal.aggarwal@pitt.edu} \\
  \And
  David Ryan Koes \\
  CMU-Pitt Program in Computational Biology\\
  Dept. of Computational \& Systems Biology\\
  University of Pittsburgh\\
  \texttt{dkoes@pitt.edu} \\
  \AND
  Nicholas M. Boffi \\
  Machine Learning Department\\
  Dept. of Mathematical Sciences\\
  Carnegie Mellon University\\
  \texttt{nboffi@andrew.cmu.edu} \\
  \And
  Eric Vanden-Eijnden \\
  Courant Institute, New York University\\
  Machine Learning Lab\\
  Capital Fund Management\\
  \texttt{eric.vanden-eijnden@cfm.com} \\
}

\begin{document}

\maketitle

\begin{abstract}

Path sampling methods generate ensembles of reactive trajectories connecting metastable states, but extracting mechanistic insight from these data remains nontrivial. We introduce \emph{Flux Matching}, a framework that learns two complementary objects directly from reactive trajectory data: a \emph{current velocity} $u(z)$, whose streamlines trace the dominant reaction pathways, and a scalar \emph{potential} $h(z)$, obtained from a weighted Helmholtz--Hodge decomposition of the reactive current, that serves as a data-driven reaction coordinate. Both minimize quadratic functionals over the reactive path ensemble, analogous to the flow matching loss in generative modeling, and require no knowledge of the underlying dynamics or stationary distribution. Unlike committor-based methods, $u$ and $h$ remain well-defined under projection onto non-Markovian collective variables, and their level sets in turn provide adaptive interfaces for improved sampling with enhanced sampling methods. Flux Matching is validated through the generation of current velocity trajectories and rate constant calculations on molecular systems.
\end{abstract}

\section{Introduction}
\label{sec:intro}

Understanding the mechanisms of rare transitions between metastable states and extreme events is a central challenge across computational chemistry, biology, materials science, and climate science. Examples range from protein folding\cite{bowman2011taming,dill2012protein} and chemical reactions to \cite{gispen2023brute} and extreme weather phenomena such as heat waves \citep{meehl2004heat, stott2004human}, sudden stratospheric warmings \citep{baldwin2001stratospheric}, and El Ni\~no onset \citep{cai2014increasing}. These events are rare because the system must traverse regions of phase space that are seldom visited under typical dynamics: a molecular dynamics trajectory may run for microseconds while the folding event of interest occurs on millisecond timescales, and obtaining robust statistics on century-scale climate extremes requires ensembles of long simulations that are computationally prohibitive.

Path sampling methods address this timescale problem by focusing computational effort on the rare fluctuations that successfully connect a reactant region $A$ to a product region $B$. Transition path sampling (TPS) \citep{bolhuis2002transition}, transition interface sampling (TIS) \citep{van2005elaborating}, forward flux sampling (FFS) \citep{allen2009forward}, and weighted ensemble (WE) \citep{huber1996weighted, aristoff2018analysis} all generate ensembles of \emph{reactive trajectories}, paths that connect $A$ to $B$, providing direct access to the statistics of transition events.

However, extracting mechanistic insight from reactive trajectory data remains nontrivial. The central object of interest is traditionally the \emph{committor function} $q(x)$, the probability that a trajectory initiated at $x$ reaches $B$ before $A$ \citep{e2006towards, vanden2006transition, hummer2004transition, best2005reaction}. The committor is often regarded as the ideal reaction coordinate, but it is fundamentally tied to the Markov property of the underlying dynamics. When trajectories are projected onto collective variables, as is inevitable in high-dimensional systems, the projected dynamics are generally non-Markovian, and the committor of the full system cannot be expressed as a function of the reduced variables alone. Methods that learn a committor in the reduced space must therefore make uncontrolled approximations.

\begin{figure*}[t]
	\centering
	\includegraphics[width=\textwidth]{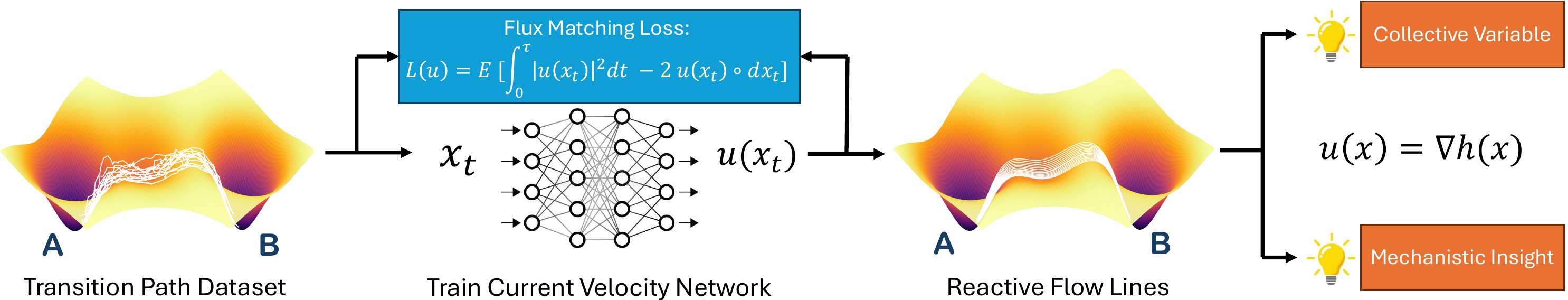}
	\caption{\textbf{Overview of Flux Matching.} From an ensemble of reactive trajectories, we learn the probability current and a scalar potential whose level sets define a data-driven reaction coordinate and milestones for refining the sampler.}
	\label{fig:overview}
\end{figure*}

In this work, we take a different approach. Rather than targeting the committor directly, we characterize the reactive path ensemble through two complementary objects, both estimated by regression from trajectory data alone.

The first is the \emph{current velocity} $u$, the ratio of the reactive current to the reactive density. It is the average instantaneous velocity of reactive trajectories, and its streamlines describe the deterministic skeleton of the transition mechanism from $A$ to $B$. The second is a scalar \emph{potential} $h$, obtained from a weighted Helmholtz--Hodge decomposition of the reactive current into an irrotational gradient component and a divergence-free remainder. The potential $h$ carries the net flux from $A$ to $B$ and serves as a learned reaction coordinate. For systems satisfying detailed balance, $h$ reduces to $\log[q/(1-q)]$ where $q$ is the forward committor, recovering the optimal committor-based coordinate without solving any boundary value problem; but the construction extends to non-Markovian settings where the committor itself is not well-defined. In particular, both $u$ and $h$ remain meaningful under projection onto collective variables, where they yield the exact marginal current and potential.

Both objects are unique minimizers of quadratic functionals over the reactive path ensemble (Theorem~\ref{thm:variational}), with empirical estimators built directly from trajectory data and requiring no knowledge of the underlying drift or stationary distribution. The loss for $u$ is structurally identical to the flow matching / stochastic interpolant loss \citep{lipman2022flow, liu2022flow, albergo2023stochastic}, with the reactive path ensemble in place of a coupling between distributions; the loss for $h$ is a Benamou--Brenier-type functional. We call this framework \emph{Flux Matching}. Beyond direct mechanistic analysis, the level sets of $h$ provide a principled scalar collective variable that can replace hand-chosen ones in adaptive samplers---interfaces for TIS or FFS, bin boundaries for WE, stopping surfaces for milestoning---enabling an iterative scheme in which improved sampling and improved estimation reinforce each other.

\textbf{Main contributions.}
\begin{itemize}
    \item \textit{Variational characterization of the reactive path ensemble.} We characterize the reactive path ensemble through a current velocity $u$ and a scalar potential $h$ obtained from a weighted Helmholtz--Hodge decomposition of the reactive current, each satisfying a quadratic variational principle (Theorem~\ref{thm:variational}) estimable by regression from trajectory data alone.

    \item \textit{Exact treatment of projected dynamics.} $u$ and $h$ remain well-defined under projection onto non-Markovian collective variables, in contrast to committor-based methods, which must assume the committor depends only on the observed variables.

    \item \textit{Data-driven reaction coordinates and sampling loop.} The streamlines of $u$ reveal the dominant transition pathways; the level sets of $h$ provide a principled reaction coordinate and adaptive milestones for TIS, FFS, WE, or milestoning, enabling iterative refinement of the sampler itself.

    \item \textit{Numerical validation.} We demonstrate the framework on the M\"uller--Brown potential under overdamped and underdamped dynamics, Alanine Dipeptide (ADP), and AIB9 molecular systems.
\end{itemize}
\subsection{Related Work}

\textbf{Path sampling methods.} Flux Matching is agnostic to how reactive trajectories are generated; suitable ensembles can be obtained from transition path sampling (TPS) \citep{bolhuis2002transition}, transition interface sampling (TIS) \citep{van2005elaborating}, forward flux sampling (FFS) \citep{allen2009forward}, weighted ensemble (WE) \citep{huber1996weighted}, or milestoning \citep{faradjian2004computing}. Adaptive multilevel splitting \citep{cerou2007adaptive, brehier2019new} and related rare event algorithms have been applied to climate and geophysical problems \citep{ragone2018computation, lucente2022coupling}; see \citet{lelievre2010free} and \citet{henin2022enhanced} for reviews.

\textbf{Flow matching and probability currents.} Flow matching \citep{lipman2023}, stochastic interpolants \citep{albergo2023building, albergo2023stochastic}, and rectified flows \citep{liu2023} learn velocity fields that transport a simple distribution to a complex target. Our variational principle (Theorem~\ref{thm:variational}) has the same quadratic structure, but the underlying measure is the reactive path ensemble rather than a coupling between distributions. Related ideas have been used to learn probability currents and entropy production rates in active matter systems \citep{boffi2024deep, boffi2025flocking, maddu2025learning}.

\textbf{Committor estimation.} Transition path theory (TPT) \citep{e2006towards, vanden2006transition, metzner2009transition} expresses the reactive density and current in terms of the committor function, and has been applied extensively to molecular simulations \citep{noe2009constructing, berezhkovskii2019single}. Likelihood-based methods \citep{peters2006obtaining, ma2005automatic} estimate the committor by fitting shooting outcomes, while neural network approaches \citep{li2019computing, khoo2019solving, rotskoff2022active, mitchell2024committor, lin2024deep} solve the committor PDE variationally. Recent work also leverages stochastic optimal(SOC) to estimate the committor\citep{du2026rare}. All of these rely on the committor being well-defined as a function of the observed state---a Markov property that fails under projection onto collective variables. Flux Matching sidesteps this issue by targeting the current directly, which remains well-defined regardless.

\section{Theoretical Framework}
%

\subsection{Reactive Trajectories}

We consider a system evolving stochastically on $\R^d$ in statistical steady state, not necessarily Markovian --- this generality covers coarse-grained or projected dynamics (Appendix~\ref{app:examples}). Transitions occur between two disjoint closed sets with smooth boundaries: the \emph{reactant} $A \subset \R^d$ and the \emph{product} $B \subset \R^d$, each possibly with multiple connected components. We denote by $\OmegaR = \R^d \setminus (A \cup B)$ the \emph{reactive region}.

\begin{pinkbox}
\begin{definition}[Reactive Trajectory and Path Ensemble]\label{def:reactive_ensemble}
A \emph{reactive trajectory} is a continuous path $\omega = (z_t)_{t \in [0, \tau(\omega)]}$ satisfying:
\begin{equation}
    z_0 \in \partial A, \qquad  z_\tau \in \partial B, \qquad z_t \in \OmegaR \quad \text{for all } t \in (0, \tau),
\end{equation}
where $\tau = \tau(\omega) > 0$ is the duration of the path, which varies from one trajectory to another. A \emph{reactive path ensemble} is a probability measure $\mathbb{P}$ on the space of such paths. We denote by $\E$ the expectation with respect to $\mathbb{P}$.
\end{definition}
\end{pinkbox}

The measure $\mathbb{P}$ encodes all statistical information about the reactive dynamics: correlations, memory effects, and the distribution of transition times $\tau$. We assume that $\mathbb{P}$-almost every path has finite quadratic variation (so Stratonovich integrals are well-defined), and that we have access to an empirical sample $\{(z^{(i)}_t)_{t \in [0, \tau_i]}\}_{i\in[N]}$.

\subsection{Reactive Density, Current, and Surface Densities}

The reactive path ensemble is characterized by four statistical quantities, defined through their duality pairings against test functions:


\newpage
\begin{pinkbox}
\begin{definition}[Reactive density, current, and surface densities]\label{def:rdc}
For any test functions $\phi: \R^d \to \R$ and $\psi: \R^d \to \R^d$,
\begin{align}
    \emph{reactive density } \rho_R: \OmegaR \to \R_{\geq 0}: \qquad & \E\left[\int_0^\tau \phi(z_t)\,dt\right] = \int_{\OmegaR} \phi(z)\,\rho_R(z)\,dz, \label{eq:density}\\
    \emph{reactive current } j_R: \OmegaR \to \R^d: \qquad & \E\left[\int_0^\tau \psi(z_t) \circ dz_t\right] = \int_{\OmegaR} \psi(z)\cdot j_R(z)\,dz, \label{eq:current}\\
    \emph{surface density on } \partial A,~\rho_{\partial A}: \qquad & \E[\phi(z_0)] = \int_{\partial A} \phi(z)\,\rho_{\partial A}(z)\,d\sigma_{\partial A}(z), \label{eq:surfdensA}\\
    \emph{surface density on } \partial B,~\rho_{\partial B}: \qquad & \E[\phi(z_\tau)] = \int_{\partial B} \phi(z)\,\rho_{\partial B}(z)\,d\sigma_{\partial B}(z), \label{eq:surfdensB}
\end{align}
where $\circ$ denotes the Stratonovich integral and $d\sigma_{\partial A}$, $d\sigma_{\partial B}$ the surface elements on $\partial A$, $\partial B$.
\end{definition}
\end{pinkbox}

Intuitively, $\rho_R$ records the time-averaged occupation of reactive trajectories --- so $\int_{\OmegaR} \rho_R\,dz = \E[\tau] \equiv T_R$, the mean transition time --- while $j_R$ records their average flux and $\rho_{\partial A}, \rho_{\partial B}$ record where they exit $A$ and enter $B$. The Stratonovich convention in~\eqref{eq:current} is natural because it respects the chain rule: for smooth paths, it reduces to $\int_0^\tau \psi(z_t) \cdot \dot z_t\,dt$. All four objects are estimable from a sample by empirical averages: e.g., $\frac{1}{N}\sum_i \int_0^{\tau_i} \phi(z^{(i)}_t)\,dt \to \int_{\OmegaR} \phi\,\rho_R\,dz$ as $N\to\infty$.

Since reactive trajectories enter $\OmegaR$ through $\partial A$ and exit through $\partial B$, the current is divergence-free in $\OmegaR$ with source and sink at its boundaries:

\begin{propbox}
\begin{lemma}\label{thm:divcondition}
The current $j_R$ satisfies
\begin{equation}
\begin{aligned}
    &\operatorname{div}\, j_R = 0 \quad && \text{in } \OmegaR,\\
    &\hat n_{\partial A} \cdot j_R = -\rho_{\partial A} \quad && \text{on } \partial A,\\
    &\hat n_{\partial B} \cdot j_R = \rho_{\partial B} \quad && \text{on } \partial B,
\end{aligned}
    \label{eq:div_free}
\end{equation}
where $\hat n_{\partial A}$ and $\hat n_{\partial B}$ denote the unit normal vectors pointing inward into $A$ and $B$, respectively.
\end{lemma}
\end{propbox}

\begin{proof}
Evaluating \eqref{eq:current} with $\psi = \nabla \phi$, the left-hand side becomes
\begin{equation}
    \label{eq:s1}
    \begin{aligned}
    \E \left[ \int_0^\tau \nabla \phi(z_t) \circ dz_t \right] &= \E[\phi(z_\tau)] - \E[\phi(z_0)]= \int_{\partial B} \phi \, \rho_{\partial B} \, d\sigma_{\partial B} - \int_{\partial A} \phi \, \rho_{\partial A} \, d\sigma_{\partial A},
    \end{aligned}
\end{equation}
where we used~\eqref{eq:surfdensA} and~\eqref{eq:surfdensB}. The right-hand side, by integration by parts, is
\begin{equation}
    \label{eq:s2}
    \int_{\OmegaR} \nabla \phi \cdot j_R \, dz = - \int_{\OmegaR} \phi \, \operatorname{div}\, j_R \, dz + \int_{\partial B} \phi \, \rho_{\partial B} \, d\sigma_{\partial B} - \int_{\partial A} \phi \, \rho_{\partial A} \, d\sigma_{\partial A}.
\end{equation}
Equating \eqref{eq:s1} and \eqref{eq:s2} yields \eqref{eq:div_free}.
\end{proof}

\subsection{Current Velocity and Potential}\label{sec: current velocity}

Given the reactive density and current, we define the \emph{current velocity} as
\begin{equation}
    u(z) = \frac{j_R(z)}{\rho_R(z)}.
    \label{eq:current_velocity}
\end{equation}
This velocity field captures the average instantaneous velocity of reactive trajectories passing through~$z$. It allows us to compute the streamlines of the probability current by integrating the \emph{probability flow ODE}
\begin{equation}
    \dot{Z}_t = u(Z_t),
    \label{eq:flow_ode}
\end{equation}
either forward or backward in time. A natural choice of initial condition is to draw $Z_0$ from the reactive trajectory ensemble, i.e., as a snapshot $z_t$ for some $t \in [0, \tau]$.

We also define the \emph{potential} $h$ via the weighted Helmholtz--Hodge decomposition of $j_R$ relative to the density $\rho_R$:
\begin{equation}
    \label{eq:helmholtz}
    j_R = \underbrace{\rho_R\,D \nabla h}_{\text{irrotational part}} + \underbrace{j_R^\perp}_{\text{solenoidal part}},
\end{equation}
where $D(z)$ is a user-specified positive-definite symmetric tensor that can be chosen to reflect the local geometry of the dynamics; in the examples below we set $D = \operatorname{Id}$. In this decomposition, the irrotational part carries the entire net flux from $A$ to $B$: for any dividing surface $\Sigma \subset \OmegaR$ separating $A$ from $B$,
\[
  \int_\Sigma j_R\cdot\hat{n}\,d\sigma = \int_\Sigma \rho_R D \nabla h\cdot\hat{n}\,d\sigma = \kappa,
\]
where $\kappa = \int_{\partial A}\rho_{\partial A}\,d\sigma_{\partial A} = \int_{\partial B}\rho_{\partial B}\,d\sigma_{\partial B}$ is the total reactive flux. Equivalently, requiring the solenoidal part $j_R^\perp$ to have zero normal flux on $\partial A \cup \partial B$ uniquely determines the decomposition, and $\rho_R D \nabla h$ then satisfies the same boundary value problem~\eqref{eq:div_free} as $j_R$. The potential $h$ thus admits a Benamou--Brenier characterization as the unique (up to a constant) minimizer of $\int_{\OmegaR}|\nabla h|^2_D\,\rho_R\,dz$ (with $|v|^2_D = v^\top D v$) subject to these prescribed boundary fluxes.
The potential $h$ thus serves as a data-driven reaction coordinate: its level sets define natural milestones for adaptive samplers such as TIS, FFS, or milestoning. For Markovian dynamics satisfying detailed balance, $h$ reduces to $\log[q/(1-q)]$, with $q$ the committor (Section~\ref{sec:tpt}); but the construction extends to non-Markovian settings where the committor itself is not well-defined.

\subsection{Variational Characterization}
The current velocity and the potential admit the following variational characterizations:
\begin{propbox}
\begin{theorem}[Flux Matching Losses] \label{thm:variational}
Given any symmetric positive-definite tensor $D(z)$, the current velocity $u$ in \eqref{eq:current_velocity} is the unique minimizer of
\begin{equation}
    \mathcal{L}_u(u) = \E \left[ \int_0^\tau |u(z_t)|_{D^{-1}(z_t)}^2 \, dt - 2 \, u(z_t)^\top D^{-1}(z_t) \circ dz_t \right],
    \label{eq:variational}
\end{equation}
and the potential $h$ defined in \eqref{eq:helmholtz} is the unique (up to a constant) minimizer of
\begin{equation}
    \mathcal{L}_h(h) = \E \left[ \int_0^\tau |\nabla h(z_t)|_{D(z_t)}^2 \, dt + 2h(z_0) - 2 h(z_\tau)\right].
    \label{eq:variational:h}
\end{equation}
\end{theorem}
\end{propbox}
 
\begin{proof}
By definition of the probability density and current of reactive trajectories, the objective \eqref{eq:variational} can be written as
\[
    \mathcal{L}_u(u) = \int_{\OmegaR} \left[ |u|_{D^{-1}}^2 \rho_R - 2\, u^\top D^{-1} j_R \right] dz,
\]
whose unique minimizer is~\eqref{eq:current_velocity}. Similarly, \eqref{eq:variational:h} reads
\[
    \mathcal{L}_h(h) = \int_{\Omega_R}|\nabla h|_D^2 \, \rho_R \, dz + 2\int_{\partial A} h \, \rho_{\partial A} \, d\sigma_{\partial A} - 2\int_{\partial B} h \, \rho_{\partial B} \, d\sigma_{\partial B},
\]
whose unique minimizer is the potential $h$ from~\eqref{eq:helmholtz}.
\end{proof}

The loss~\eqref{eq:variational} is a $D^{-1}$-weighted version of the flow matching / stochastic interpolants objective \citep{albergo2023building,albergo2023stochastic}, using reactive trajectories as the interpolants; a related formulation was used in \citep{boffi2025flocking} to characterize probability currents and entropy production rates in nonequilibrium systems. The two losses in Theorem~\ref{thm:variational} are intimately connected: $\mathcal{L}_h$ is the restriction of $\mathcal{L}_u$ to gradient fields. Indeed, substituting $u = D\nabla h$ in \eqref{eq:variational}, the weighted norm becomes $|D\nabla h|^2_{D^{-1}} = |\nabla h|^2_D$, and the Stratonovich integrand reduces to $(D\nabla h)^\top D^{-1}\circ dz_t = \nabla h^\top \circ dz_t = dh(z_t)$, whose integral telescopes to $h(z_\tau) - h(z_0)$; this yields \eqref{eq:variational:h} exactly.

\subsection{Connection to Transition Path Theory}\label{sec:tpt}

When the dynamics is Markovian, our framework recovers the objects from transition path theory (TPT). Consider an SDE
\begin{equation}\label{eq:sde}
    dx_t = b(x_t)\,dt + \sqrt{2}\,\sigma(x_t)\,dW_t,
\end{equation}
with stationary density $\rho$ and diffusion tensor $D = \sigma\sigma^\top$, and let $q(x)$ and $\tilde q(x)$ be the forward and backward committors --- the probabilities of reaching $B$ before $A$, and of having come from $A$ rather than $B$, respectively. As shown in Appendix~\ref{app:tpt},
\begin{equation}\label{eq:rhoj_tpt}
    \rho_R = \rho\, q\,\tilde q, \qquad j_R = q\tilde q\, j + \rho\bigl[\tilde q\, D\nabla q - q\,D\nabla\tilde q\bigr],
\end{equation}
where $j = b\rho - \operatorname{div}(D\rho)$ is the stationary current of~\eqref{eq:sde}. The current velocity follows as
\begin{equation}\label{eq:u_tpt}
    u = b - \operatorname{div} D - D\nabla\log\rho + D\nabla\log q - D\nabla\log\tilde q.
\end{equation}
Under detailed balance, $j = 0$ and $\tilde q = 1 - q$, so~\eqref{eq:u_tpt} simplifies to $u = D\nabla\log[q/(1-q)]$; using the diffusion tensor $D$ from~\eqref{eq:sde} in~\eqref{eq:variational:h} then gives $h = \log[q/(1-q)]$, the optimal committor-based reaction coordinate. Concrete examples of reactive path ensembles (including overdamped and underdamped Langevin dynamics, and the non-Markovian case of position-only observation) are given in Appendix~\ref{app:examples}.

\subsection{Non-Markovian Projection}
\label{sec:nonmarkov}

When a system is observed through a projection $z_t = \phi(x_t)$ for some $\phi: \R^d \to \R^n$ with $n < d$, the dynamics is generally non-Markovian and the committor cannot be expressed as a function of the reduced variables alone. Nevertheless, the current velocity remains well-defined.

\begin{propbox}
\begin{theorem}\label{thm:projected_current}
Let $z_t = \phi(x_t)$ where $\phi: \R^d \to \R^n$ is $C^2$ with full rank Jacobian and $x_t$ solves the SDE~\eqref{eq:sde}. The current velocity in the reduced space is
\begin{equation}
\label{eq:u_projected}
    u(z) = \E\big[ \left(b(x) + 2D(x)\nabla\log q(x) \right) \cdot \nabla\phi(x)  + D(x) : \nabla\nabla\phi(x) \,\big|\, \phi(x) = z \big],
\end{equation}
where the expectation is over~$\rho_R$ conditional on~$\phi(x) = z$.
\end{theorem}
\end{propbox}

The proof, given in Appendix~\ref{app:proof}, follows from applying It\^o's lemma to $z_t = \phi(x_t^R)$, where $x_t^R$ evolves according to the Doob-transformed dynamics
\begin{equation}\label{eq:reactive_sde}
    dx_t^R = \left[b(x_t^R) + 2D(x_t^R)\nabla\log q(x_t^R)\right] dt + \sqrt{2}\,\sigma(x_t^R) \, dW_t,
\end{equation}
which governs trajectories conditioned on being reactive. When $\phi = \mathrm{id}$, this recovers~\eqref{eq:u_tpt}. For general $\phi$, expression~\eqref{eq:u_projected} does not require solving for the committors explicitly: it is the conditional expectation of $\dot z_t$ along reactive trajectories, which Flux Matching learns directly from data.

\section{Implementation}
\label{sec:methods}

We minimize the losses of Theorem~\ref{thm:variational} by stochastic gradient descent on minibatches of reactive trajectories; the procedure is summarized in Algorithm~\ref{alg:flux_matching}. Training $h_\phi$ requires differentiating through the inner gradient $\nabla h_\phi(z_k)$ that appears in the loss and therefore relies on double backpropagation, which is supported by modern autograd frameworks.

\begin{algorithm}[t]
\caption{Flux Matching for the current velocity $u$ and/or the potential $h$}
\label{alg:flux_matching}
\DontPrintSemicolon
\KwIn{Reactive trajectories $\{(z^{(m)}_t)_{t \in [0,\tau_m]}\}_{m=1}^M$, time lag $\Delta t$, learning rate $\eta$, neural network $u_\theta$ with initial parameters $\theta_0$ and/or neural network $h_\phi$ with initial parameters $\phi_0$, diffusion tensor $D(z)$ (e.g.\ $D = \operatorname{Id}$).}
\For{$\textrm{iteration} = 1, 2, \ldots$}{
    Sample minibatch $\mathcal{B} \subset \{1, \ldots, M\}$\;
    $\mathcal{K} \leftarrow \{k \in \mathcal{B} : \tau_k > 2\Delta t\}$\;
    \For{$k \in \mathcal{K}$}{
        Sample $t_k \sim \mathcal{U}([\Delta t,\, \tau_k - \Delta t])$\;
        Set $z_k = z^{(k)}_{t_k}$,\ \ $\Delta z_k = \tfrac{1}{2}\bigl(z^{(k)}_{t_k+\Delta t} - z^{(k)}_{t_k-\Delta t}\bigr)$,\ \ $D_k = D(z_k)$\;
    }
    $\displaystyle L_u = \frac{1}{|\mathcal{K}|}\sum_{k \in \mathcal{K}}\Bigl(\bigl|u_\theta(z_k)\bigr|_{D_k^{-1}}^2\,\Delta t \;-\; 2\,u_\theta(z_k)^\top D_k^{-1}\,\Delta z_k\Bigr)$\quad and/or\;
    $\displaystyle L_h = \frac{1}{|\mathcal{K}|}\sum_{k \in \mathcal{K}}\Bigl(\bigl|\nabla h_\phi(z_k)\bigr|_{D_k}^2\,(\tau_k - 2\Delta t)\Bigr) + \mathcal{L}_{\mathrm{CE}}(h_\phi)$\;
    $\theta \leftarrow \theta - \eta\,\nabla_\theta L_u$\quad and/or\quad $\phi \leftarrow \phi - \eta\,\nabla_\phi L_h$\;
}
\Return $u_\theta$ and/or $h_\phi$\;
\end{algorithm}

\subsection{The time lag $\Delta t$}

The time lag $\Delta t$ controls a bias--variance tradeoff in the Stratonovich approximation. Small $\Delta t$ gives a more accurate increment but yields noisier displacements $\Delta z_k$, since the diffusive component dominates on short timescales; large $\Delta t$ reduces this noise but introduces bias when $u$ varies on scales below $\Delta t$. In practice, $\Delta t$ should be small compared to the timescale on which $u$ varies, yet large enough that displacements remain well-resolved relative to noise. The centered form $\Delta z_k = \tfrac{1}{2}(z^{(k)}_{t_k+\Delta t} - z^{(k)}_{t_k-\Delta t})$ approximates the Stratonovich integral to second order. One can also sample $\Delta t$ randomly within a range or anneal it during training---starting larger for stability and decreasing for accuracy.

\subsection{Boundary regularization for the potential}

The bulk term $|\nabla h|^2_D$ in $\mathcal{L}_h$ is well-behaved, but the boundary terms $2h(z_0) - 2h(z_\tau)$ have constant gradients that push $h$ toward $-\infty$ on $\partial A$ and $+\infty$ on $\partial B$. This is consistent with the theoretical limit $h = \log[q/(1-q)]$ under detailed balance but is a recipe for exploding gradients in practice. We therefore replace these boundary terms by a bounded logistic surrogate, the cross-entropy loss
\begin{equation}
    \mathcal{L}_{\mathrm{CE}}(h) = -\frac{1}{|\mathcal{K}|}\sum_{k \in \mathcal{K}}\Bigl[\log\sigma\bigl(-h(z^{(k)}_0)\bigr) + \log\sigma\bigl(h(z^{(k)}_{\tau_k})\bigr)\Bigr],
    \label{eq:cross_entropy}
\end{equation}
where $\sigma$ denotes the sigmoid. The minimization direction is the same as the original boundary penalty, with $h(z_0)$ driven low and $h(z_\tau)$ driven high, but the gradients saturate as $|h|$ grows, preventing runaway updates. The loss minimized in practice is the bulk term plus $\mathcal{L}_{\mathrm{CE}}$, as shown in Algorithm~\ref{alg:flux_matching}.

\subsection{Neural network parameterization}

We parameterize the velocity field and the potential as neural networks $u_\theta: \R^n \to \R^n$ and $h_\phi: \R^n \to \R$. For molecular systems, the inputs are either the full 3D atomic coordinates or internal dihedral angles; in the Cartesian setting, translational and rotational symmetries are handled via mean-centering and rotational data augmentation. Full architectural details are deferred to \cref{app: model archs}.

\begin{figure}[t]
    \centering
	\begin{subfigure}{0.23\textwidth}
		\includegraphics[width=\linewidth]{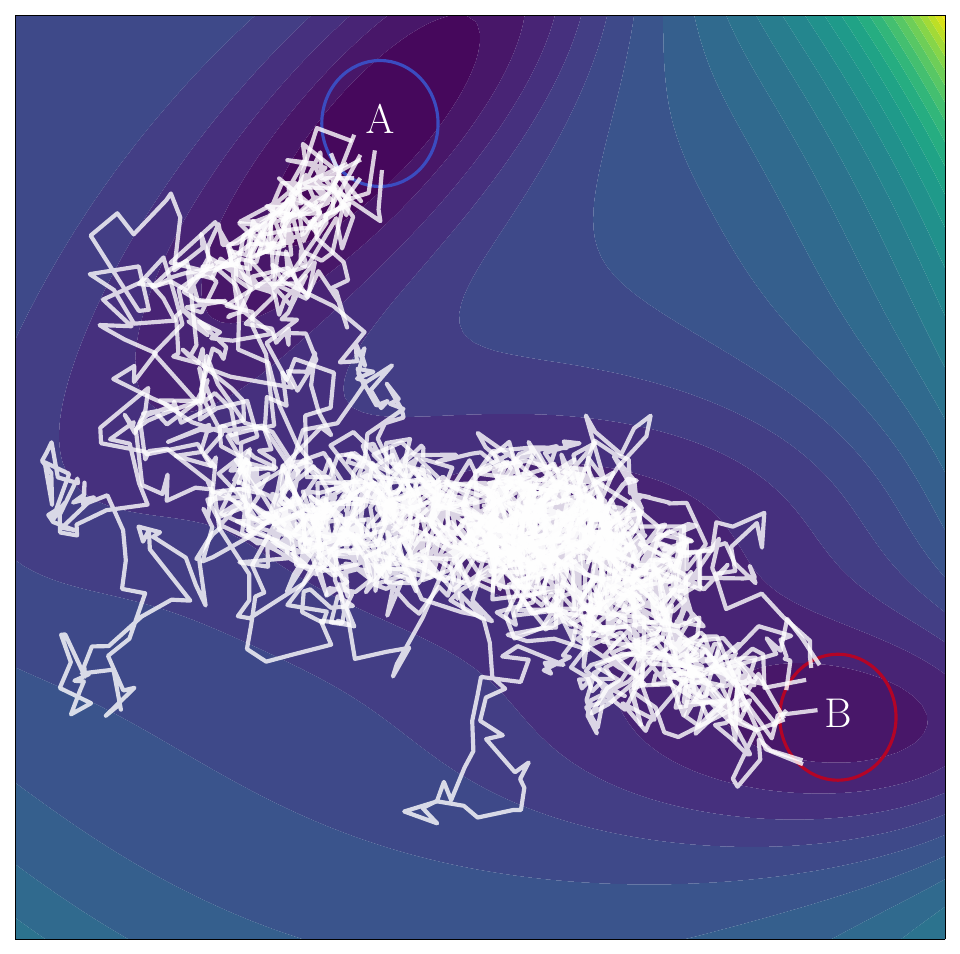}
	\end{subfigure}
    \begin{subfigure}{0.23\textwidth}
		\includegraphics[width=\linewidth]{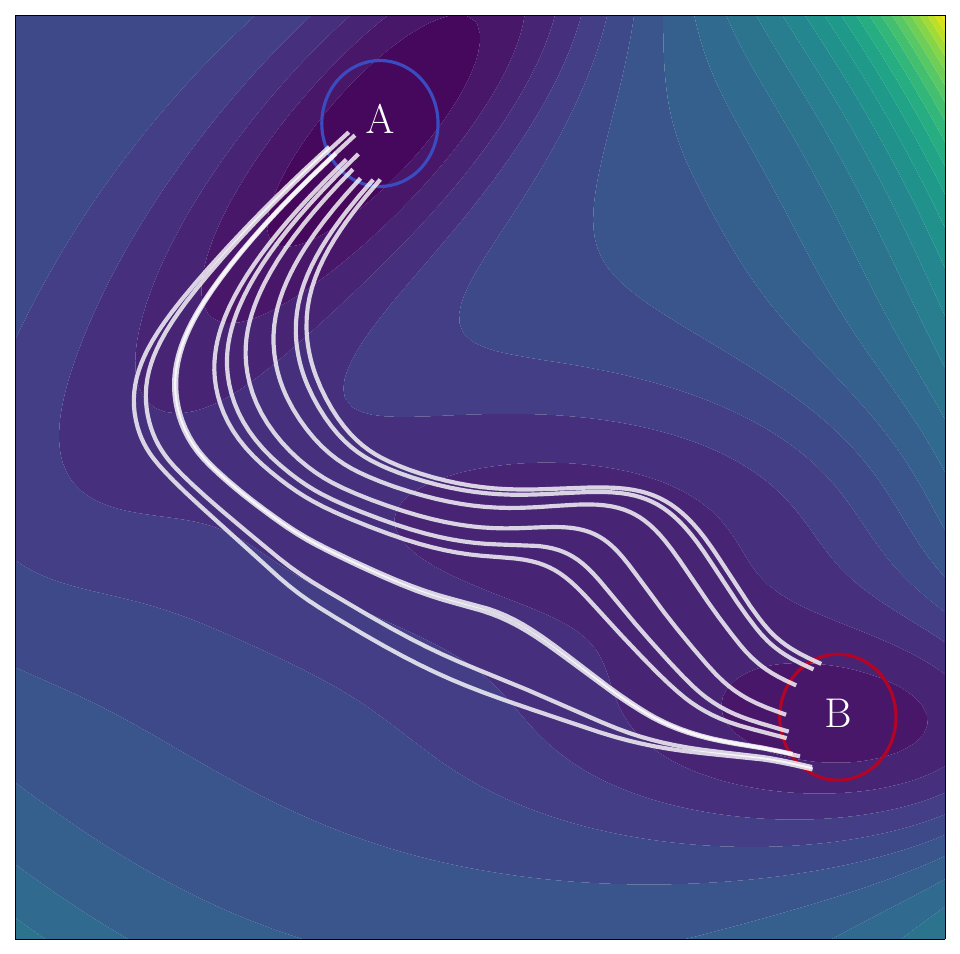}
	\end{subfigure}
    \begin{subfigure}{0.23\textwidth}
		\includegraphics[width=\linewidth]{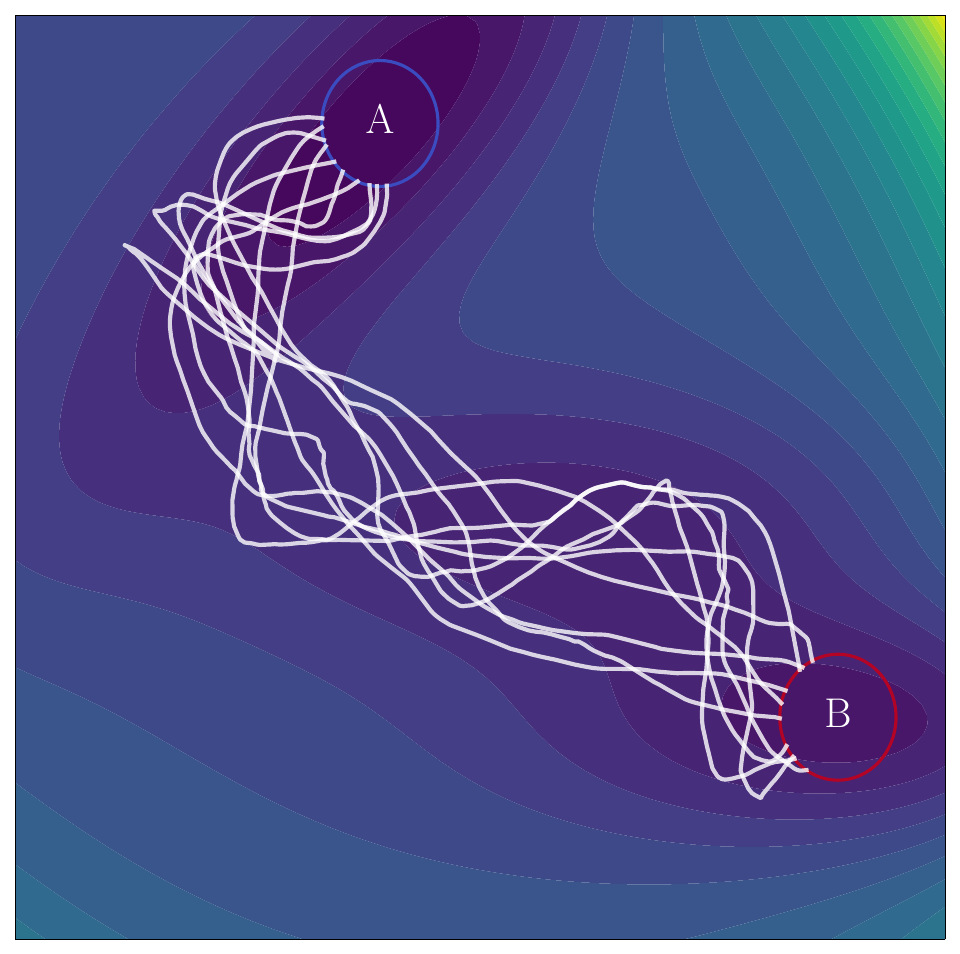}
	\end{subfigure}
    \begin{subfigure}{0.23\textwidth}
		\includegraphics[width=\linewidth]{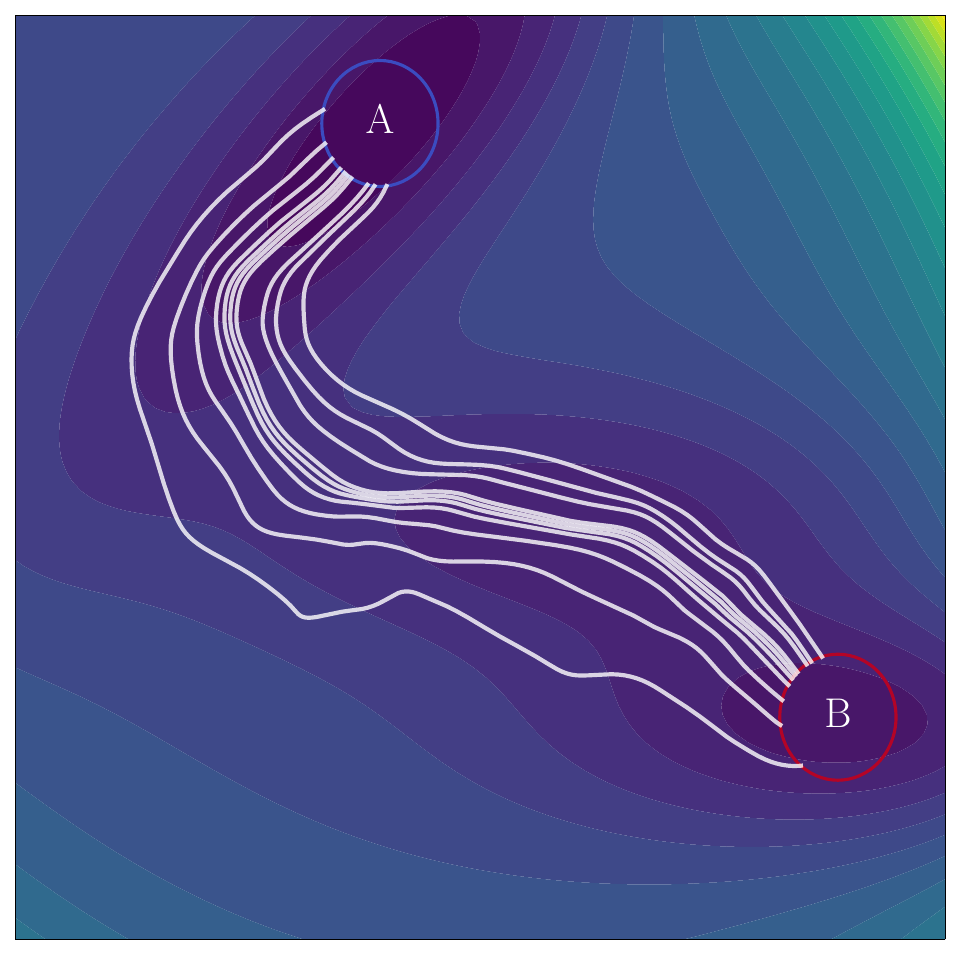}
	\end{subfigure}
	\caption{\textbf{Transition path and current trajectories on the Müller-Brown potential.} Reactive trajectories and learned current-velocity streamlines for overdamped (first pair) and underdamped (second pair) dynamics. The streamlines smoothly track the reactive channel in both regimes; underdamped paths show greater spread due to inertia.}
	\label{fig:mueller-brown}
\end{figure}



\section{Numerical Experiments}
\label{sec:experiments}

Flux Matching is investigated in several different settings. First, we demonstrate the effectiveness of Flux Matching on the toy 2D Müller-Brown (MB) system with both overdamped (first-order) and underdamped (second-order) dynamics and is shown in Figure~\ref{fig:mueller-brown}. We further conduct experiments on the higher-dimensional Alanine Dipeptide (ADP) and AIB9 molecular systems.


\begin{figure}[t]
    \centering
	\begin{subfigure}{0.29\textwidth}
		\includegraphics[width=\linewidth]{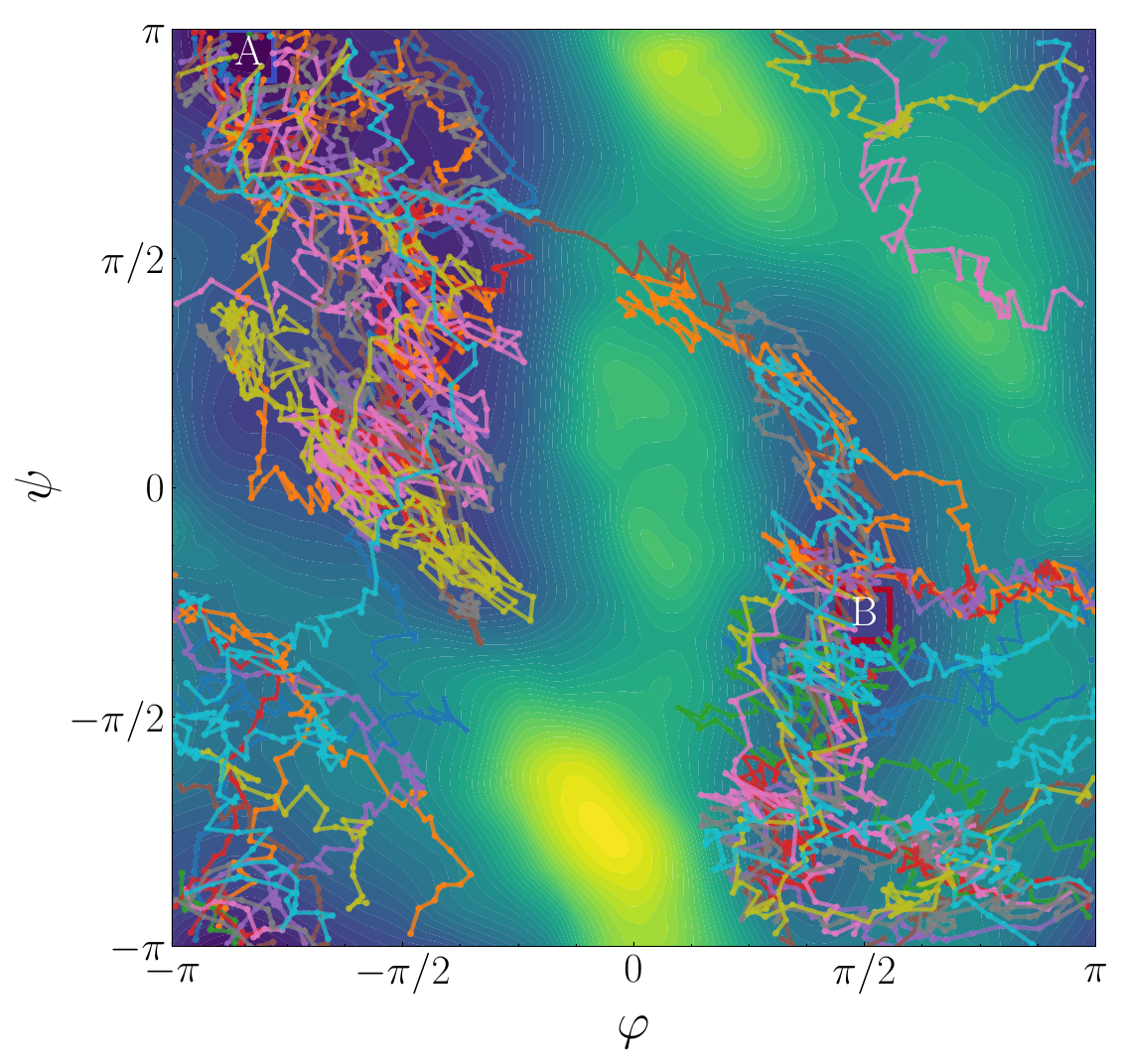}
	\end{subfigure}
    \begin{subfigure}{0.29\textwidth}
		\includegraphics[width=\linewidth]{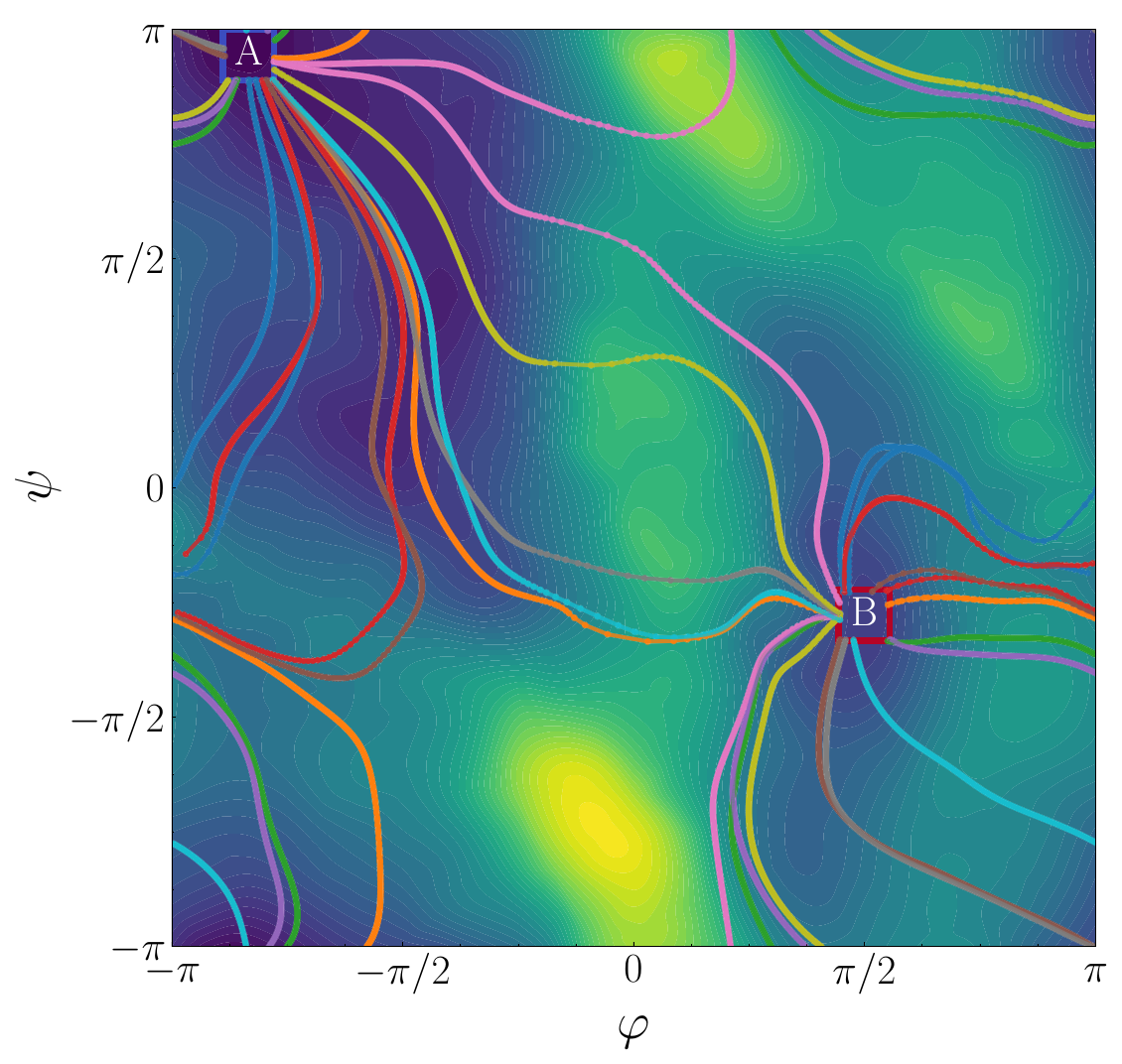}
	\end{subfigure}
    \begin{subfigure}{0.32\textwidth}
		\includegraphics[width=\linewidth]{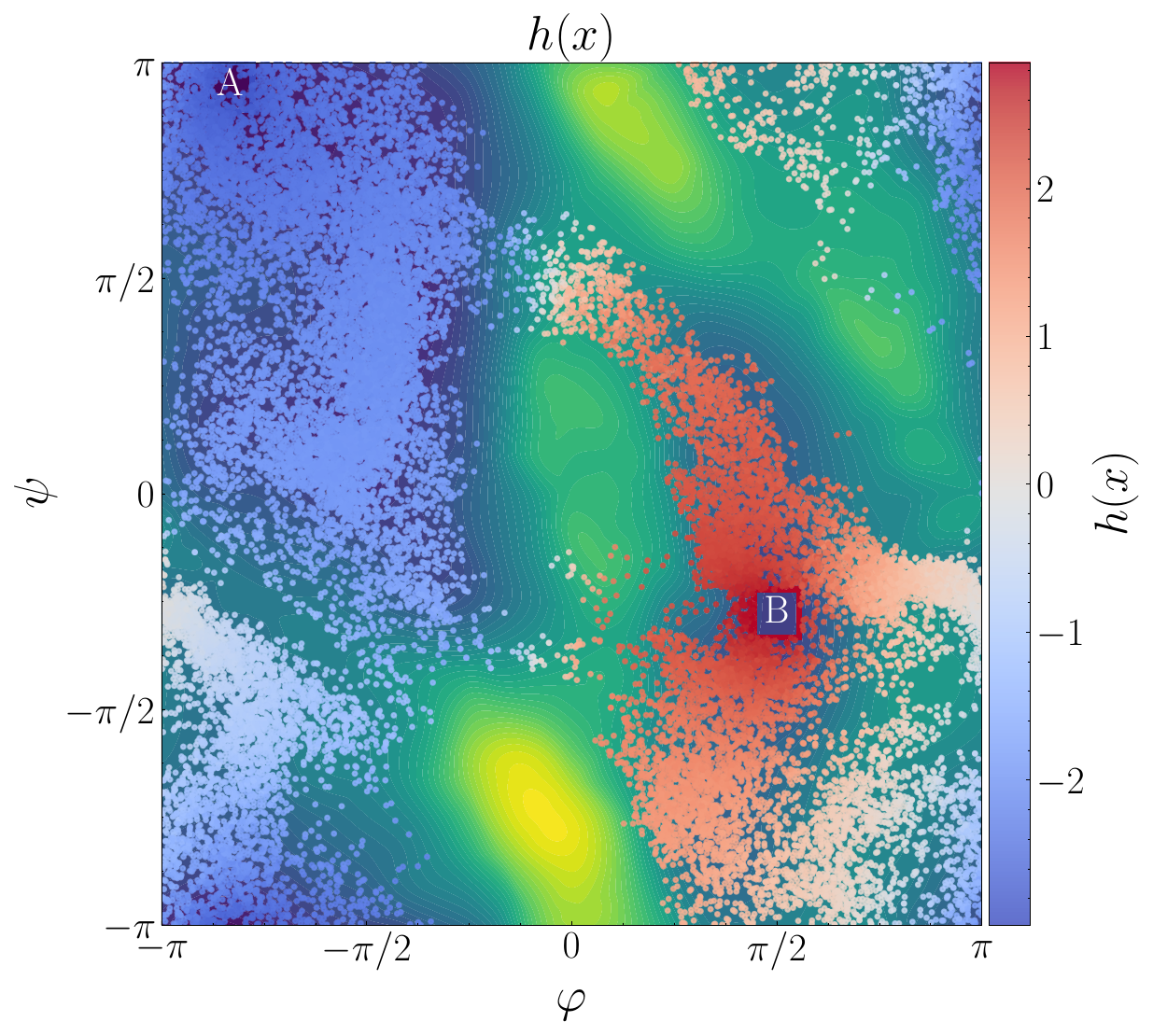}
	\end{subfigure}
    
	\caption{\textbf{Transition path, current trajectories and potential $h$ on ADP.} Reactive trajectories (left), integrated streamlines of the learned current velocity (center), and the learned potential $h(x)$ (right), projected onto the backbone dihedral angles $(\varphi, \psi)$. The streamlines reliably reach $B$ from initial conditions in $A$ and resolve the two main reaction channels; the potential $h$ varies monotonically along the reaction, providing a one-dimensional reaction coordinate.}
	\label{fig:alanine-ramachandran}
\end{figure}

\subsection{Dataset Generation and Evaluation Metrics}

The datasets for all systems in this study are generated using Transition Path Sampling (TPS) with one-way shooting moves. Briefly, shooting moves perform MCMC sampling in path space to obtain an ensemble of reactive trajectories under a given system's dynamics. Initial trajectories are generated at high temperature, followed by a brief equilibration period, after which the production ensemble of reactive paths is collected. System-specific details are deferred to~\cref{app: test systems}.

Once integrated, the learned current velocity should produce paths whose statistics match those of the reactive path ensemble. We therefore evaluate generated trajectories using the following metrics, in addition to visualizing the level sets of $h$:

\begin{itemize}
    \item \textbf{Completion rate:} Completion rate measures the fraction of flow lines, that reach $\partial A$, and $\partial B$ on reverse and forward integration within a fixed number of integration steps from a point randomly sampled from the TPS ensemble. A score of 1 indicates an accurately learned current-velocity field.
    \item \textbf{Torsional Wasserstein ($\mathbb{T} - W_2$):} The torsional W-2 distance between samples generated from current trajectories and compared to samples from the TPS ensemble. Lower values indicate a more accurately learned current.  $\mathbb{T}-W_2$ is computed over batches of 10,000 samples.
\end{itemize}

\begin{table}[tbh]
	\caption{\textbf{Quantitative results across systems.} Values are computed over 256 trajectories. In both systems, the generated trajectories closely reproduce the values observed in the reactive trajectory ensemble.}
	\label{tab:results}
	\centering
	\begin{tabular}{lcccccc}
		\toprule
		System & Method & Features & Completion rate & $\mathbb{T}$-$W_2$\\
		\midrule
		ADP & TPS & $\mathbb{R}^{3N}$ & 1 & $0.0875 \pm 0.006$\\
        ADP & FM & $\mathbb{R}^{3N}$ & 0.7734 & $0.6453 \pm 0.016$\\
        ADP & FM & Dihedrals & 0.9804 & $0.7220 \pm 0.017$\\
        \midrule
        AIB9 & TPS & $\mathbb{R}^{3N}$ & 1 & $0.1011 \pm 0.034$ \\
        AIB9 & FM & Backbone Dihedrals & 0.7578 & $0.9137 \pm 0.007$\\
		\bottomrule
	\end{tabular}
\end{table}

\subsection{Alanine Dipeptide}
\label{sec:alanine-dipeptide}

Alanine Dipeptide (N-acetyl-L-alanine-N'-methylamide) is a canonical benchmark system for testing rare event methods in molecular simulation \citep{bolhuis2000reaction, ma2005automatic, ren2005transition}. Despite its small size---22 atoms and 66 degrees of freedom in Cartesian coordinates---the molecule exhibits conformational transitions between metastable states that share qualitative features with larger biomolecular systems. The free energy landscape is typically characterized using two backbone dihedral angles, $\phi$ (C-N-C$_\alpha$-C) and $\psi$ (N-C$_\alpha$-C-N), which provide a natural low-dimensional projection \citep{ramachandran1963stereochemistry}. 

We study the transition between the main metastable states, $C_7^{\text{eq}}$ (equatorial) and $C_7^{\text{ax}}$ (axial) \citep{tobias1993molecular, smith1999alanine}. Reactive trajectories are generated using transition path sampling in vacuum at 300\,K using OpenPathSampling \citep{swenson2019openpathsampling1, swenson2019openpathsampling2}.\footnote{\url{http://openpathsampling.org}}. We train separate models on the full Cartesian space $\mathbb{R}^{3N}$ space, and on the full set of 7 dihedral angles of the molecule. 

Table~\ref{tab:results} reports the performance of both representations. The dihedral representation generally yields stronger results, primarily due to its lower dimensionality. This also highlights an advantage of Flux Matching: it remains exact even under non-Markovian projections (Section~\ref{sec:nonmarkov}).

Figure~\ref{fig:alanine-ramachandran} shows the TPS reactive trajectories (left) and the streamlines of the learned current velocity (center), both projected onto the ($\varphi, \psi$) dihedral plane. The projected flow lines yield smooth transitions between states and are visibly more interpretable than the raw reactive path ensemble. Figure~\ref{fig:alanine-ramachandran} additionally displays the value of $h$ on snapshots from reactive trajectories. The potential $h$ increases monotonically along the reaction channels connecting the two states, indicating that it captures the relevant motions of the reaction and serves as a natural collective variable for further analysis. Example transition-path structures colored by $h$, together with rate constants computed from weighted-ensemble simulations using $h$ as the collective variable, are provided in~\cref{sec: structures,app:adp_rate}.

\subsection{AIB9 molecular system}
\label{sec: AIB9}

\begin{figure}[t]
    \centering
	\begin{subfigure}{0.29\textwidth}
		\includegraphics[width=\linewidth]{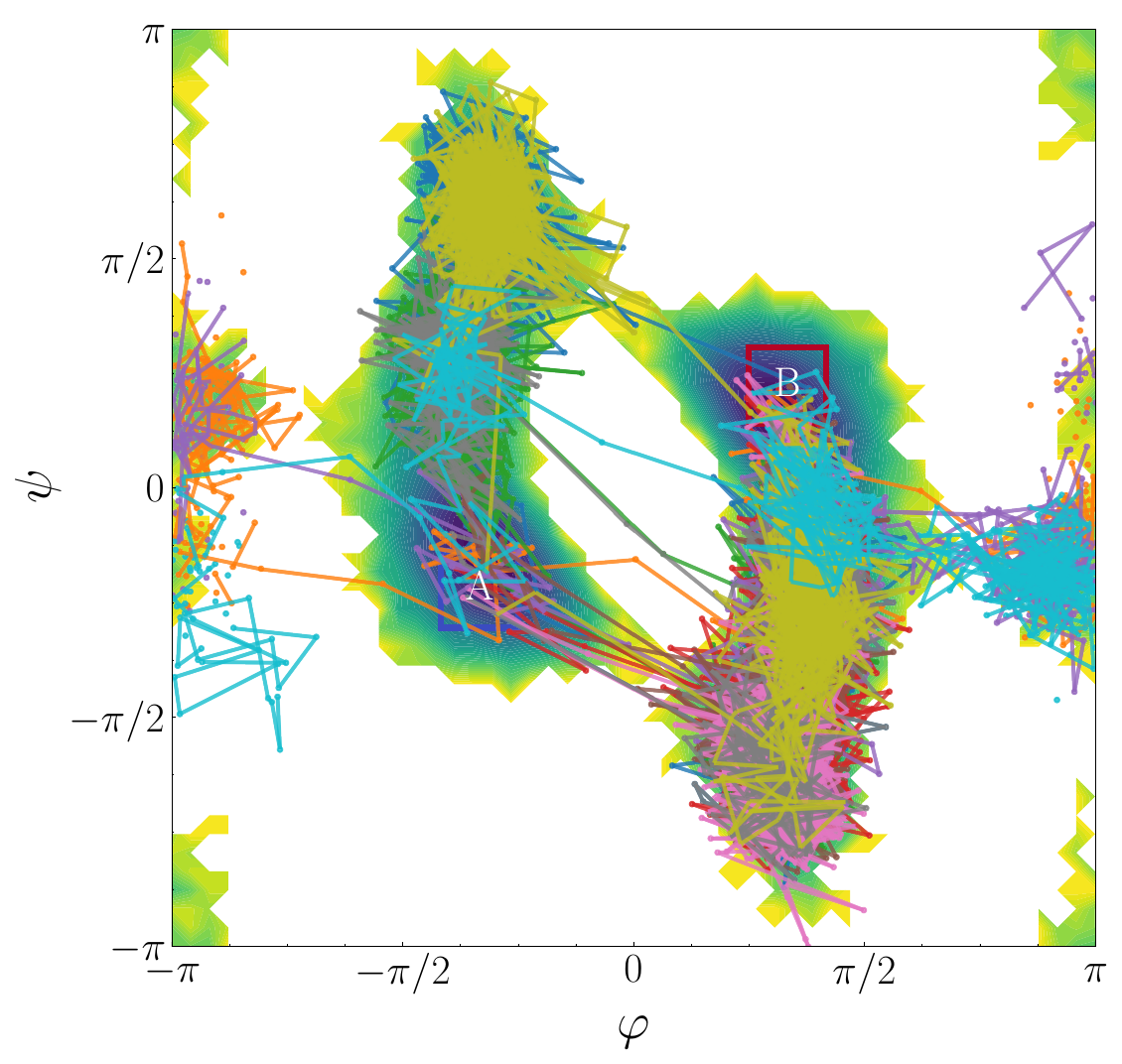}
	\end{subfigure}
    \begin{subfigure}{0.29\textwidth}
		\includegraphics[width=\linewidth]{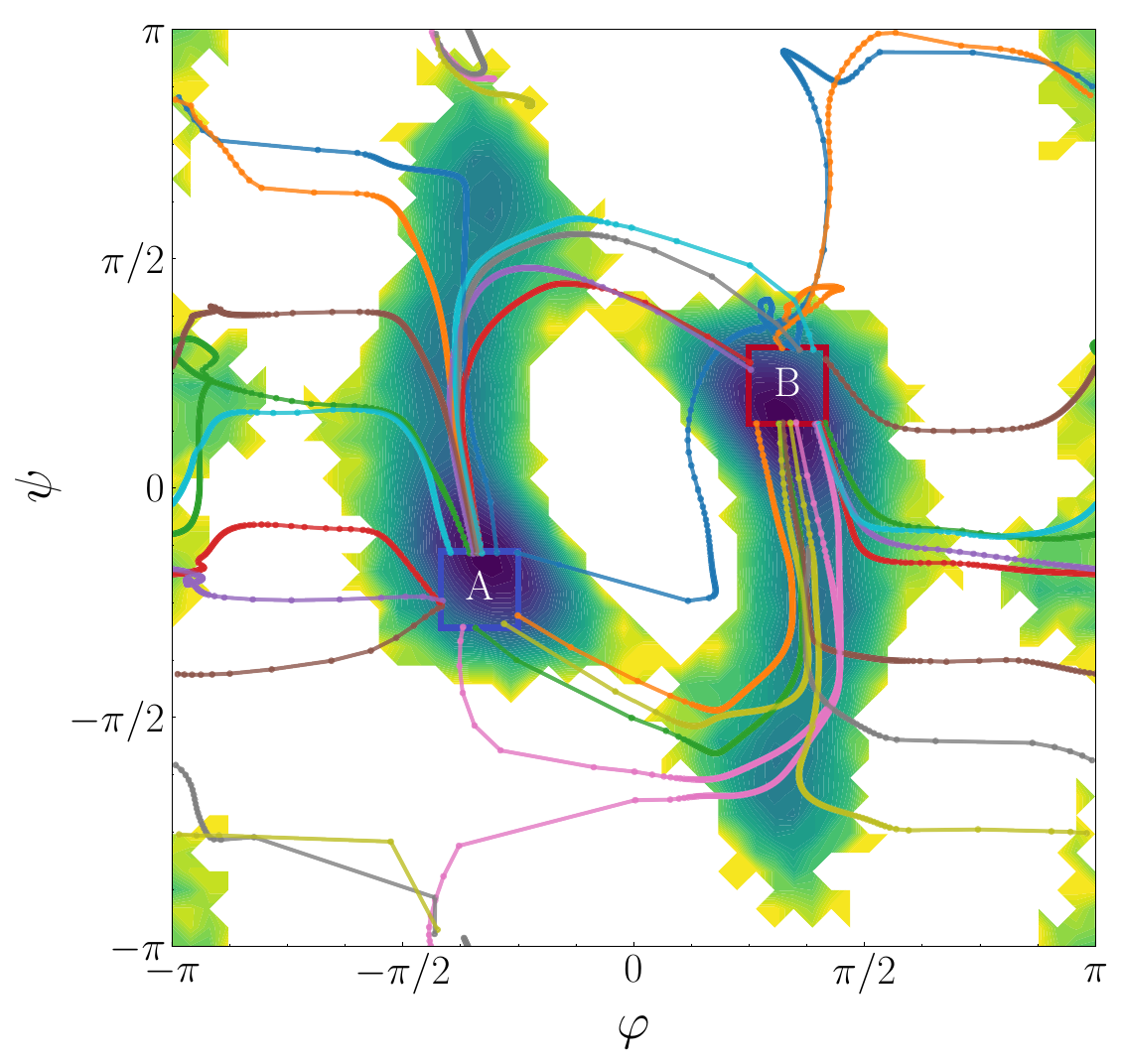}
	\end{subfigure}
    \begin{subfigure}{0.32\textwidth}
		\includegraphics[width=\linewidth]{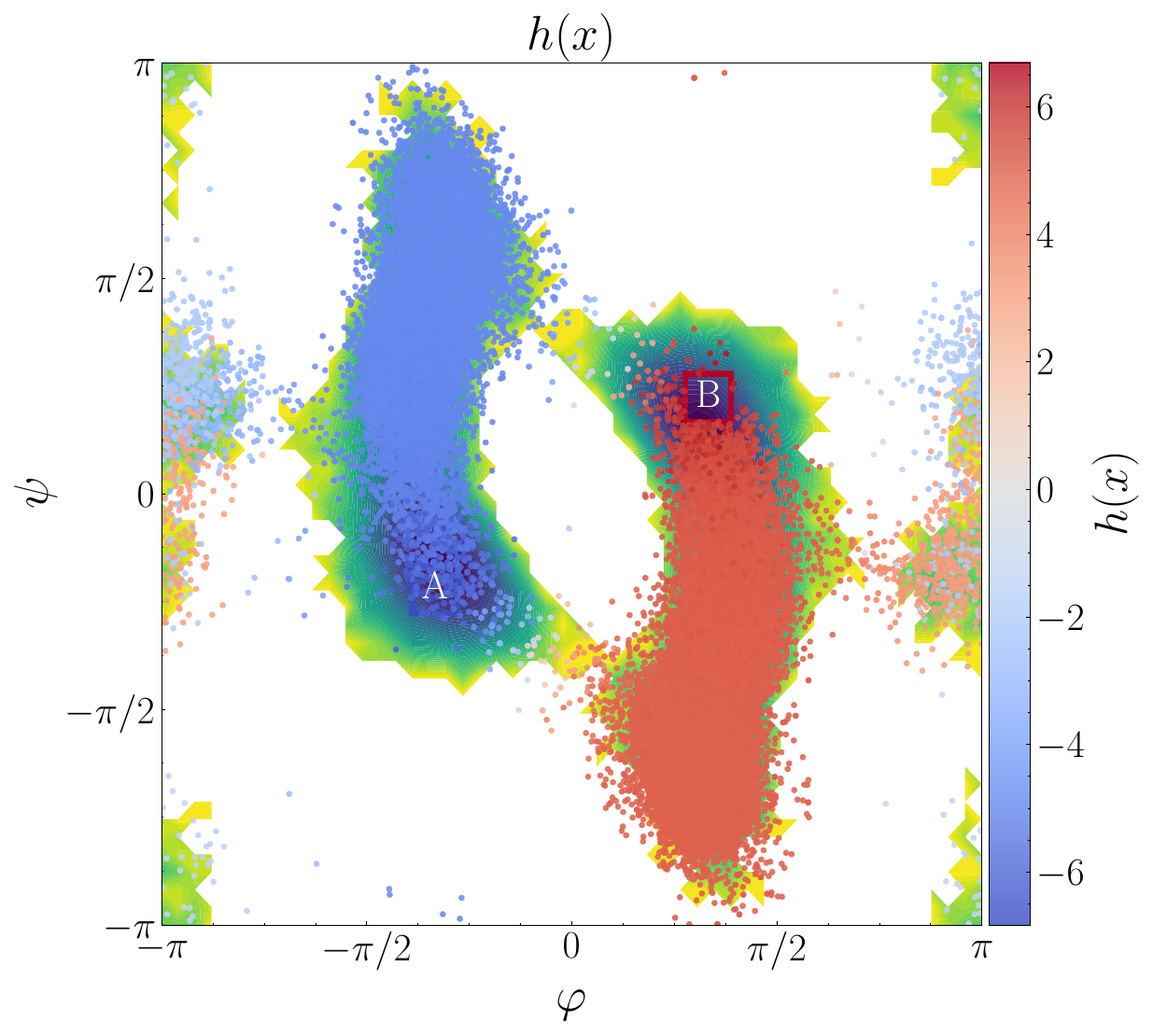}
	\end{subfigure}
    
	\caption{\textbf{Transition paths, current trajectories, and potential $h$ for AIB9.} Reactive trajectories (left), integrated streamlines of the learned current velocity (center), and the reactive potential (right), all projected onto the backbone dihedral angles. The streamlines reliably reach $B$ from initial conditions in $A$, and the potential $h$ varies monotonically along the reaction, providing a one-dimensional reaction coordinate.}
	\label{fig:AIB9_ramachandran}
\end{figure}

We next consider the synthetic AIB9 peptide. This system is more challenging, as it exhibits intermediate metastable states that complicate the learning of both the reactive current and the collective variable. Reactive trajectories are generated with TPS in vacuum at 500\,K using OpenPathSampling. The system is also higher-dimensional, with 129 atoms; we therefore operate directly on the 29 backbone dihedral angles.

As shown in Table~\ref{tab:results}, a large fraction of the generated trajectories successfully connect the two states through appropriate reaction channels. Figure~\ref{fig:AIB9_ramachandran} shows the reactive path ensemble (left), generated flow lines (center), and learned collective variable $h$ (right), projected onto the dihedral angles of the central residue. Reactive-trajectory snapshots are saved every 100\,fs, producing visibly large jumps in this projection, whereas the flow lines exhibit smooth transitions between states---further supporting our claim that the learned currents yield more interpretable reaction mechanisms. As with ADP, $h$ varies monotonically along the reaction direction. Additional structures colored by $h$ are provided in~\cref{sec: structures}.

\section{Conclusion}
\label{sec:conclusion}

We have presented Flux Matching, a framework for learning the probability current of reactive trajectories directly from path sampling data. The method is agnostic to how reactive trajectories are generated: it can be applied to ensembles from transition path sampling, transition interface sampling, forward flux sampling, weighted ensemble, or any other method that produces reactive paths. This flexibility allows Flux Matching to leverage well-established sampling techniques while using modern flow matching tools to estimate the current velocity field. By operating directly in the full state space, the method bypasses the need for collective variables, which are often difficult to identify in complex systems. For systems satisfying detailed balance, the learned current connects to the committor function through the gradient structure $u \propto \nabla \log [q_+/(1-q_+)]$; for general non-Markovian systems, the current velocity remains well-defined even when the committor is not. The learned current provides both quantitative information---transition rates, milestone placement---and qualitative mechanistic insight into reaction pathways and transition state regions.

Scaling the method to larger systems---from complex biomolecular dynamics to atmosphere-ocean flows---remains an important direction for future work. This will require neural network architectures capable of handling high-dimensional state spaces, as well as strategies to improve the quality of the reactive trajectory ensemble. A natural approach is to iterate between path sampling and current estimation: the milestones extracted from a learned current can be used to define interfaces for transition interface sampling or bins for weighted ensemble, which in turn generate trajectories that improve the current estimate. Such an iterative scheme would combine the strengths of classical rare event methods with the representational power of modern generative models.

\section{Acknowledgements}
\label{sec:acknowledgements}

We thank Carles Domingo-Enrich and Jonathan Weare for valuable discussions on the project.

This work is funded through R35GM140753 from the National Institute of General Medical Sciences. The content is solely the responsibility of the authors and does not necessarily represent the official views of the National Institute of General Medical Sciences or the National Institutes of Health.

\bibliographystyle{unsrtnat}
\bibliography{reference}

\newpage
\appendix





\section{Connection to Transition Path Theory}\label{app:tpt}

In this appendix, we show that when the dynamics is Markovian, our framework recovers the objects from transition path theory (TPT). Consider an SDE of the form
\begin{equation}\label{eq:sde_app}
    dz_t = b(z_t) \, dt + \sqrt{2} \, \sigma(z_t) \, dW_t,
\end{equation}
with stationary density $\rho(z)$ and diffusion tensor $D(z) = \sigma(z)\sigma^\top(z)$.

\begin{proposition}[Forward and backward committors]\label{prop:committors}
Let $q(z)$ denote the probability of reaching $B$ before $A$ starting from $z$, and $\tilde q(z)$ the probability of having come from $A$ rather than $B$. These satisfy the boundary value problems
\begin{align}
    0 &= b \cdot \nabla q + D : \nabla\nabla q, \label{eq:forward_committor_app}\\
    0 &= (-b + 2D\nabla \log \rho) \cdot \nabla \tilde q + D : \nabla\nabla \tilde q, \label{eq:backward_committor_app}
\end{align}
with boundary conditions $q|_{\partial A} = \tilde q|_{\partial B} = 0$ and $q|_{\partial B} = \tilde q|_{\partial A} = 1$.
\end{proposition}

\begin{proof}
The forward committor $q(z) = \mathbb{P}_z(\tau_B < \tau_A)$ satisfies the backward Kolmogorov equation for the process \eqref{eq:sde_app}. Let $\mathcal{L} = b \cdot \nabla + D : \nabla\nabla$ denote the generator. For $z \in \Omega_R$, the function $q$ is harmonic with respect to $\mathcal{L}$, i.e., $\mathcal{L} q = 0$, with the stated boundary conditions.

The backward committor $\tilde q(z)$ is the probability that the time-reversed process came from $A$. Under time reversal, the SDE \eqref{eq:sde_app} with stationary measure $\rho$ becomes
\begin{equation}
    d\tilde z_t = \tilde{b}(\tilde z_t) \, dt + \sqrt{2} \, \sigma(\tilde z_t) \, dW_t,
\end{equation}
where the process $\tilde z_t$ has the same law as $z_{-t}$ and the reversed drift is $\tilde{b} = -b + 2D\nabla\log\rho$. The backward committor satisfies the backward Kolmogorov equation for this reversed process, giving \eqref{eq:backward_committor_app}.
\end{proof}

\begin{proposition}[Reactive density and current]\label{prop:reactive_density_current}
The reactive density and current are given by
\begin{align}
    \rho_R &= C_R^{-1} \rho \, q \, \tilde q, \label{eq:rho_tpt_app}\\
    j_R &= C_R^{-1} \rho \left[ \tilde q \, D \nabla q - q \, D \nabla \tilde q \right], \label{eq:j_tpt_app}
\end{align}
where $C_R = \int_{\Omega_R} \rho(z) \, q(z) \, \tilde q(z) \, dz$.
\end{proposition}

\begin{proof}
Using the Markovian property, a point $z \in \Omega_R$ belongs to a reactive trajectory if and only if the trajectory will reach $B$ before $A$ (probability $q(z)$) and came from $A$ rather than $B$ (probability $\tilde q(z)$). The reactive density is therefore proportional to $\rho(z) q(z) \tilde q(z)$, giving \eqref{eq:rho_tpt_app}.

For the current, we use the probability flux formula. The flux of the stationary process is $j = \rho b - D\nabla\rho$. The reactive current counts only trajectories that are reactive, weighted by the probability of being so. Using the product rule and the committor equations \eqref{eq:forward_committor_app}--\eqref{eq:backward_committor_app}, one can verify that \eqref{eq:j_tpt_app} satisfies $\mathrm{div}\, j_R = 0$ with the correct boundary flux, and that $j_R = \rho_R u$ with $u$ given below.
\end{proof}

\begin{proposition}[Current velocity]\label{prop:current_velocity}
The current velocity is
\begin{equation}\label{eq:u_tpt_app}
    u(z) = D(z) \nabla \log q(z) - D(z) \nabla \log \tilde q(z).
\end{equation}
For systems satisfying detailed balance, $\tilde q = 1 - q$, and this simplifies to
\begin{equation}
    u(z) = D(z) \nabla \log \frac{q(z)}{1-q(z)f}fd.
\end{equation}
\end{proposition}

\begin{proof}
Dividing \eqref{eq:j_tpt_app} by \eqref{eq:rho_tpt_app}:
\begin{align}
    u = \frac{j_R}{\rho_R} &= \frac{\tilde q D \nabla q - q D \nabla \tilde q}{q \tilde q} \\
    &= D \frac{\nabla q}{q} - D \frac{\nabla \tilde q}{\tilde q} \\
    &= D \nabla \log q - D \nabla \log \tilde q.
\end{align}
For detailed balance, the stationary measure is $\rho \propto e^{-\beta U}$ and the drift satisfies $b = -D\nabla U + \nabla \cdot D$. In this case, the forward and backward committors are related by $\tilde q = 1 - q$, which gives the stated simplification.
\end{proof}

\begin{proposition}[Doob transform]\label{prop:doob}
Reactive trajectories, i.e., trajectories of \eqref{eq:sde_app} conditioned on reaching $B$ before $A$, evolve according to
\begin{equation}
\label{eq:reactive_sde_app}
\begin{aligned}
    dz_t^R &= b(z_t^R) \, dt + 2D(z_t^R)\nabla\log q(z_t^R) \, dt + \sqrt{2}\,\sigma(z_t^R) \, dW_t.
\end{aligned}
\end{equation}
\end{proposition}

\begin{proof}
This is the Doob $h$-transform with $h = q$. For a diffusion with generator $\mathcal{L}$, conditioning on reaching a set where $h = 1$ before a set where $h = 0$ (with $\mathcal{L}h = 0$ in between) transforms the drift by adding $2D\nabla\log h$. Since $q$ is harmonic for $\mathcal{L}$ with $q|_{\partial B} = 1$ and $q|_{\partial A} = 0$, the conditioned process has drift $b + 2D\nabla\log q$.
\end{proof}

\section{Examples of reactive path ensembles}
\label{app:examples}
Here, we describe some canonical examples of reactive path ensembles that arise in applications from chemistry and physics.

\paragraph{Overdamped Langevin Dynamics.}

Consider the overdamped Langevin equation
\begin{equation}
    dx_t = -\nabla U(x_t) \, dt + \sqrt{2\beta^{-1}} \, dW_t,
    \label{eq:overdamped}
\end{equation}
where $U : \R^d \to \R$ is a $C^2$ potential such that $Z = \int_{\R^d} e^{-\beta U(x)} dx < \infty$, and $\beta = 1/(k_B T)$ is the inverse temperature with $k_B$ the Boltzmann constant and $T$ the temperature. The invariant density is $\rho(x) = Z^{-1} e^{-\beta U(x)}$.

Taking $z_t = x_t$, the reactive density and current are given by
\begin{align}
    \rho_R(x) &= C_R^{-1} \, e^{-\beta U(x)} q(x) (1 - q(x)), \label{eq:rho_overdamped} \\
    j_R(x) &= C_R^{-1} \beta^{-1} e^{-\beta U(x)} \nabla q(x), \label{eq:j_overdamped}
\end{align}
where $C_R = \int_{\R^d} e^{-\beta U(x)} q(x)(1 - q(x)) \, dx$ is a normalization constant, and $q(x)$ is the \emph{committor function}: the probability that the solution to \eqref{eq:overdamped}, started from $x$, reaches $B$ before $A$. It satisfies the backward Kolmogorov equation
\begin{equation}
    -\nabla U(x) \cdot \nabla q(x) + \beta^{-1} \Delta q(x) = 0,
    \label{eq:bke_overdamped}
\end{equation}
with boundary conditions $q|_{\partial A} = 0$ and $q|_{\partial B} = 1$. Note that because detailed balance holds, the backward committor is simply $1 - q(x)$.

The current velocity is therefore
\begin{equation}
    u(x) = \frac{\beta^{-1} \nabla q(x)}{q(x)(1 - q(x))}.
    \label{eq:u_overdamped}
\end{equation}

\paragraph{Underdamped Langevin Dynamics.}

Consider the underdamped Langevin equation
\begin{align}
    dx_t &= v_t \, dt, \label{eq:underdamped_x} \\
    dv_t &= -\nabla U(x_t) \, dt - \gamma v_t \, dt + \sqrt{2\beta^{-1} m^{-1} \gamma} \, dW_t, \label{eq:underdamped_v}
\end{align}
where $\gamma > 0$ is the friction coefficient and $m$ is the (diagonal) mass matrix. The invariant density is $\rho(x,v) = Z_H^{-1} e^{-\beta H(x,v)} $, where $H(x,v) = U(x) + \frac{1}{2} v^\top m \, v$ is the Hamiltonian and $Z_H = \int_{\mathbb{R}^d \times \mathbb{R}^d} e^{-\beta H(x,v)} dx dv$ is the partition function.

Taking $z_t = (x_t, v_t)$, the reactive density and current are given by
\begin{align}
    \rho_R(x, v) &= C_R^{-1} e^{-\beta H(x,v)} q(x, v) (1 - q(x, -v)),
    \label{eq:rho_underdamped}\\
    j_R(x, v) &= C_R^{-1} e^{-\beta H(x,v)} \begin{pmatrix}
        v\\ -\nabla U(x) - g_R(x,v)
    \end{pmatrix}
    \label{eq:j_underdamped}
\end{align}
where the normalization constant is
\begin{equation}
    C_R = \int_{\R^d\times\R^d} e^{-\beta H(x,v)} q(x, v) (1 - q(x, -v)) \, dx \, dv,
\end{equation}
and we defined
\begin{equation}
    \label{eq:g:def}
    \begin{aligned}
    g_R(x,v) = \gamma \beta^{-1} m^{-1} [ &(1-q(x, -v)) \nabla_v q(x, v)- q(x,v) \nabla_v q(x, -v)].
    \end{aligned}
\end{equation}
The committor $q(x, v)$ now depends on both position and velocity, and satisfies the backward Kolmogorov equation
\begin{equation}
    v \cdot \nabla_x q - \nabla U(x) \cdot \nabla_v q - \gamma v \cdot \nabla_v q + \beta^{-1} m^{-1} \gamma \Delta_v q = 0,
    \label{eq:bke_underdamped}
\end{equation}
with boundary conditions $q|_{\partial A} = 0$ and $q|_{\partial B} = 1$.

\begin{remark}
The factor $(1 - q(x, -v))$ in \eqref{eq:rho_underdamped}, with velocity reversed, reflects the time-reversal symmetry of reactive trajectories: under time reversal, velocities change sign, so a trajectory reactive from $A$ to $B$ with velocity $v$ corresponds to a trajectory from $B$ to $A$ with velocity $-v$.
\end{remark}

\paragraph{Non-Markovian Case: Position-Only Observation.}

The framework applies equally to non-Markovian processes. As an example, consider the underdamped Langevin dynamics \eqref{eq:underdamped_x}--\eqref{eq:underdamped_v}, but now take $z_t = x_t$ (position only), marginalizing over the velocity. The reactive density becomes
\begin{equation}
    \rho_R(x) = C_R^{-1} e^{-\beta U(x)} \int e^{-\frac{\beta}{2} v^\top m v} q(x, v) (1 - q(x, -v)) \, dv,
    \label{eq:rho_nonmarkov}
\end{equation}
and the current velocity is
\begin{equation}
    u(x) = \frac{\int v \, e^{-\frac{\beta}{2} v^\top m v} q(x, v) (1 - q(x, -v)) \, dv}{\int e^{-\frac{\beta}{2} v^\top m v} q(x, v) (1 - q(x, -v)) \, dv}.
    \label{eq:u_nonmarkov}
\end{equation}

\begin{remark}
Equations \eqref{eq:rho_nonmarkov}--\eqref{eq:u_nonmarkov} are exact expressions for the reactive density and current velocity of the non-Markovian process $z_t = x_t$. No approximation is involved; the velocity marginalization is built into Definition~\ref{def:rdc}.
\end{remark}

\section{Proof of Theorem~\ref{thm:projected_current}}
\label{app:proof}

\begin{proof}
Reactive trajectories in the full space evolve according to the Doob-transformed dynamics \eqref{eq:reactive_sde}:
\begin{equation}
\begin{aligned}
    dx_t^R &= b(x_t^R) \, dt + 2D(x_t^R)\nabla\log q(x_t^R) \, dt+ \sqrt{2}\,\sigma(x_t^R) \, dW_t.
\end{aligned}
\end{equation}
Applying Itô's lemma to $z_t = \phi(x_t^R)$, we obtain
\begin{equation}
    dz_t = \tilde{b}(x_t^R) \, dt + \sqrt{2}\,\nabla\phi(x_t^R)^\top \sigma(x_t^R) \, dW_t,
\end{equation}
where 
\begin{equation}
\begin{aligned}
    \tilde{b}(x) &= \left( b(x) + 2D(x)\nabla\log q(x) \right) \cdot \nabla\phi(x) + D(x) : \nabla\nabla\phi(x).
\end{aligned}
\end{equation}

To identify the current velocity, we use Definition~\ref{def:rdc} with a test function $\psi: \R^n \to \R^n$. The Stratonovich integral satisfies
\begin{equation}
    \int_0^\tau \psi(z_t) \circ dz_t = \int_0^\tau \psi(z_t) \cdot dz_t + \frac{1}{2} \int_0^\tau d\langle \psi(z), z \rangle_t,
\end{equation}
where the first term on the right is an Itô integral. The quadratic covariation is
\begin{equation}
    d\langle \psi(z), z \rangle_t = 2 \nabla\psi(z_t) : \nabla\phi(x_t^R)^\top D(x_t^R) \nabla\phi(x_t^R) \, dt.
\end{equation}
Substituting the SDE for $dz_t$ into the Itô integral:
\begin{equation}
\begin{aligned}
    \int_0^\tau \psi(z_t) \cdot dz_t &= \int_0^\tau \psi(z_t) \cdot \tilde{b}(x_t^R) \, dt + \sqrt{2} \int_0^\tau \psi(z_t) \cdot \nabla\phi(x_t^R)^\top \sigma(x_t^R) \, dW_t.
\end{aligned}
\end{equation}
The stochastic integral is a martingale, so its expectation vanishes. Therefore,
\begin{equation}
\begin{aligned}
    \E\left[ \int_0^\tau \psi(z_t) \circ dz_t \right] &= \E\left[ \int_0^\tau \psi(z_t) \cdot \tilde{b}(x_t^R) \, dt \right] + \E\left[ \int_0^\tau \nabla\psi(z_t) : \nabla\phi(x_t^R)^\top D(x_t^R) \nabla\phi(x_t^R) \, dt \right].
\end{aligned}
\end{equation}
By Definition~\ref{def:rdc}, this equals
\begin{equation}
    T_R \int_{\R^n} \left[ \psi(z) \cdot u(z) + \nabla\psi(z) : D_z(z) \right] \rho_R(z) \, dz,
\end{equation}
where $D_z(z) = \E\left[ \nabla\phi(x)^\top D(x) \nabla\phi(x) \,|\, \phi(x) = z \right]$ is the projected diffusion tensor. Comparing with Definition~\ref{def:rdc}, the term involving $\nabla\psi$ cancels between the Stratonovich correction and the projected diffusion, and we identify
\begin{equation}
\begin{aligned}
    u(z) &= \E\big[ \left(b(x) + 2D(x)\nabla\log q(x) \right) \cdot \nabla\phi(x) + D(x) : \nabla\nabla\phi(x) \,\big|\, \phi(x) = z \big],
\end{aligned}
\end{equation}
where the expectation is over $\rho_R(x)$ conditional on $\phi(x) = z$.
\end{proof}

\section{Test system}\label{app: test systems}

\subsection{Müller-Brown}\label{app: Muller-Brown}

The Müller-Brown system follows the energy function given by:

\begin{equation}
    \begin{aligned}
    U(x, y) = & -200 \cdot \exp \left( -(x-1)^2 -10 y^2 \right) \\
              & -100 \cdot \exp \left( -x^2 - 10 \cdot (y - 0.5)^2 \right) \\
              & -170 \cdot \exp \left( -6.5 \cdot (0.5 + x)^2 + 11 \cdot (x +0.5) \cdot (y -1.5) -6.5 \cdot (y -1.5)^2 \right) \\
              & + 15 \cdot \exp \left( 0.7 \cdot (1 + x)^2 +0.6 \cdot (x + 1) \cdot (y -1) +0.7 \cdot (y -1)^2 \right) \, .
    \end{aligned}
\end{equation}

The basins A and B are circles of radius 0.1 centered at [-1.5, 0.9] and [-0.5, 1.7] respectively

Overdamped Langevin dynamics was simulated using an Euler-Maruyama integrator, while underdamped dynamics was simulated using the BAOAB integrator. Both dynamics used a time step of $1e^-4$ and $\beta^{-1} = 12.5$. The underdamped dynamics also used a friction coefficient of $\gamma = 10$. 

TPS shooting moves were implemented in a straightforward manner. Shooting points were selected uniformly along each trajectory, and the shooting direction was randomly chosen. If the backward direction was chosen, velocities were reversed. If a reactive trajectory was generated from the shooting point, it was accepted or rejected based on the Metropolis criterion, which is given by the ratio of the new and old transition path lengths. A total of 1,000 transition paths were generated for both overdamped and underdamped dynamics

\subsection{Alanine Dipeptide}\label{app: ADP}

Transition paths for Alanine Dipeptide were generated using the OpenPathSampling Python package with the OpenMM\cite{eastman2023openmm} engine and the LangevinMiddleIntegrator. The system was simulated with the CHARMM27 force field at a time step of 1 fs, friction coefficient of 1/ps and hydrogen bond restaints. Molecular states are stored every 10 fs of simulation. Basins A and B follow the same definitions as in \cite{mitchell2024committor} with squares of side $10^\circ$ centered at $[\varphi,\psi]=[-150^\circ,170^\circ]$ and $[90^\circ,50^\circ]$, respectively. A total of 100,000 transition paths were generated for this system. We conduct a dataset ablation with dataset sizes of both 10,000 and 100,000 trajectories.

\subsection{AIB9}

Similar to ADP, the reactive path ensemble for AIB9 is also generated using the OpenPathSampling Python package with the OpenMM engine and the LangevinMiddleIntegrator. The system is simulated with the Amber15-IPQ force field\cite{bogetti2020twist} in vaccuum at 500K with time steps of 1fs, friction coefficients of 1/ps and Hbond constraints. Basins A and B are defined as squares of side $10^\circ$ centered at $[\varphi,\psi]=[-60^\circ,-40^\circ]$ and $[60^\circ,40^\circ]$, respectively. Simulation frames are saved every 100fs. We generate a dataset of 100,000 paths from AIB9 as well.

\section{Experimental Details}\label{app: exp detail}

\subsection{Model Architecture}\label{app: model archs}

\paragraph{Muller Brown:} Current velocity model for the Muller Brown is a simple MLP model with 5 layers, hidden dimension of 128 and SiLU activation functions

\paragraph{ADP $\mathbb{R}^{3N}$: } Alanine Dipeptide is featurized through a simple atom typing scheme where all the atom types are made distinguishable. The velocity model is parameterized as a non-equivariant transformer. Therefore the node features also contain an embedding of the (mean-centered) position.

The attention coefficients $A$ are calculated as

\begin{equation}
	\label{graphformer_head}
	A
	= \sigma \!\Bigl(\frac{Q K^{\mathsf{T}}}{\sqrt{d}} + B\Bigr)
\end{equation}

where $B$ is a learnable bias computed from pairs of atoms. This pairwise attention bias takes as input the atom types and the interatomic distance. The overall architecture is heavily inspired by the Graphormer architecture \cite{shi2022benchmarking}. The transformer consists of 4 self-attention layers, each with 16 attention heads. The hidden dimension is 512, and layer normalization is applied both before and after the attention blocks. The final potential prediction is produced by an MLP applied to a sum of the hidden node features.

\paragraph{ADP and AIB9 Dihedrals.}
The dihedral angles for ADP and AIB9 are bounded in $[-\pi, \pi)$ and exhibit periodic structure that must be respected by the input representation. To handle this, we lift each dihedral $x$ into a continuous, differentiable embedding by concatenating $[\sin(x),\, \sin(x + \pi/4)]$. This featurization removes the discontinuity at the $\pm\pi$ boundary and ensures non-vanishing gradients across the full angular domain, since the two sinusoids are phase-shifted so that their derivatives never vanish simultaneously. Both dihedral models are 4-layer MLPs with SiLU activations.

\subsection{Training Details}\label{app:training_details}

\paragraph{Training data and time-step convention.}
Across all systems, the current-velocity models are trained on triplets $(x_{t-\Delta t},\, x_t,\, x_{t+\Delta t})$ sampled from transition-path trajectories. The training time step $\Delta t$ is chosen such that the central finite-difference velocity estimate
\[
\frac{x_{t+\Delta t} - x_{t-\Delta t}}{2\Delta t} \;\sim\; \mathcal{O}(1)
\]
is well-scaled for optimization, avoiding the vanishing- and exploding-target regimes that arise for very small or very large $\Delta t$. All models are optimized with AdamW. At the end of each epoch, we sample 100 initial points from the reactive ensemble and integrate the learned velocity field both forward and backward in time to generate full trajectories (\cref{app:inference_details}); the checkpoint achieving the highest fraction of trajectories that successfully connect the two endpoint states is retained for inference. Below, we report only the system-specific hyperparameters that deviate from this shared protocol.

\paragraph{Müller--Brown.}
We use $\Delta t = 10^{-4}$ and train for 25 epochs at a constant learning rate of $10^{-3}$.

\paragraph{ADP, Cartesian coordinates ($\mathbb{R}^{3N}$).}
Triplets are drawn from saved trajectories with consecutive configurations separated by $10$~fs (every other frame in the saved trajectories), and we set $\Delta t = 10^{-3}$ to satisfy the scaling condition above. Training runs for 400 epochs, with a linear learning-rate warmup to $5\times 10^{-4}$ over 5{,}000 steps followed by cosine annealing to $10^{-5}$ over 100{,}000 steps.

\paragraph{ADP, dihedral coordinates.}
Configurations are again separated by $10$~fs, and we find that $\Delta t = 10^{-1}$ satisfies the scaling condition. Training runs for 1{,}000 epochs at a constant learning rate of $5\times 10^{-4}$, with a \texttt{ReduceLROnPlateau} scheduler that decays the learning rate by a factor of 0.1 whenever the validation loss has not improved for 30 consecutive epochs.

\paragraph{AIB9, backbone dihedrals.}
Configurations are separated by $100$~fs; all remaining hyperparameters match the ADP dihedral setup above.

\subsection{Generation of Flow Lines}\label{app:inference_details}
To generate a flow line, we sample an initial point $x_0$ from the reactive ensemble, $x_0 \sim \rho_{\mathcal{R}}$, and integrate the learned velocity field independently in the forward and backward directions:
\begin{equation}
x^{+}_{s} = x_0 + \int_0^{s} u(x^{+}_{r})\,dr, \quad s \in [0, \tau_+],
\qquad
x^{-}_{s} = x_0 - \int_0^{s} u(x^{-}_{r})\,dr, \quad s \in [0, \tau_-],
\end{equation}
where the forward and backward stopping times are
\begin{equation}
\tau_+ := \min\bigl\{ s \ge 0 \;\big|\; x^{+}_{s} \in \partial B \bigr\} \wedge T_{\max},
\qquad
\tau_- := \min\bigl\{ s \ge 0 \;\big|\; x^{-}_{s} \in \partial A \bigr\} \wedge T_{\max}.
\end{equation}
The forward integration terminates when the trajectory first enters the product boundary $\partial B$, and the backward integration terminates when it first enters the reactant boundary $\partial A$; in either case, integration also halts if the maximum time $T_{\max}$ is exceeded. The two segments are concatenated to form a complete flow line connecting $\partial A$ to $\partial B$ through $x_0$. The boundary shells $\partial A$ and $\partial B$ used to terminate inference-time integration are taken larger than those used to define the training data, since the learned velocity is not well-defined at the boundaries of $A$ and $B$.

\paragraph{Müller--Brown.}
Integration uses an explicit Euler scheme with step size $\Delta t = 10^{-4}$.

\paragraph{ADP, Cartesian coordinates ($\mathbb{R}^{3N}$).}
Integration uses a fourth-order Runge--Kutta solver with $\Delta t = 10^{-4}$. After each step, a post-processing routine enforces the same hydrogen-bond constraints applied during MD data generation, ensuring consistency with the constrained dynamics and improving sample quality in practice.

\paragraph{ADP and AIB9, dihedral coordinates.}
Integration uses a fourth-order Runge--Kutta solver with $\Delta t = 10^{-2}$. After each step, dihedral angles are wrapped back into $[-\pi, \pi)$ to preserve the periodic structure of the coordinate space.

\subsection{Metrics}\label{sec:metrics}

We evaluate the quality of generated flow lines along two complementary axes: whether trajectories successfully traverse the transition, and how closely the resulting ensemble of configurations matches the reference distribution of reactive paths.

\paragraph{Completion rate.}
A generated flow line is considered \emph{complete} if it connects the reactant and product states within the maximum integration time $T_{\max}$, i.e., if both stopping times $\tau_+$ and $\tau_-$ are attained before reaching $T_{\max}$. The completion rate is the fraction of sampled initial points $x_0 \sim \rho_{\mathcal{R}}$ for which this condition is met.

\paragraph{Torsional Wasserstein-2 distance.}
To assess the distributional fidelity of completed flow lines, we compute the Wasserstein-2 ($\mathcal{W}_2$) distance between the empirical distribution of generated configurations $\mu$ and a reference distribution $\nu$ drawn from MD reactive trajectories, with both distributions represented through backbone torsion angles. Because torsions live on the periodic domain $[-\pi, \pi)$, we use the circular ground distance
\[
d_{\mathrm{circ}}(\theta, \theta') := \min\bigl(\, |\theta - \theta'|,\; 2\pi - |\theta - \theta'| \,\bigr),
\]
extended coordinate-wise to the product of circles $\mathbb{T}^d$, and define the torsional $\mathcal{W}_2$ as
\[
\mathcal{W}_2^{\mathrm{tors}}(\mu, \nu) := \left( \inf_{\gamma \in \Gamma(\mu, \nu)} \int d_{\mathrm{circ}}(\theta, \theta')^2 \, d\gamma(\theta, \theta') \right)^{1/2},
\]
where $\Gamma(\mu, \nu)$ denotes the set of couplings between $\mu$ and $\nu$. The infimum is computed exactly by solving the discrete optimal-transport linear program with the POT library~\citep{flamary2021pot}. Lower $\mathcal{W}_2^{\mathrm{tors}}$ values indicate that the model reproduces the geometry of the reactive ensemble more faithfully.

\section{Additional Results}

\subsection{Example Transition path structures}
\label{sec: structures}

\begin{figure}[!h]
    \centering
	\includegraphics[width=0.5\linewidth]{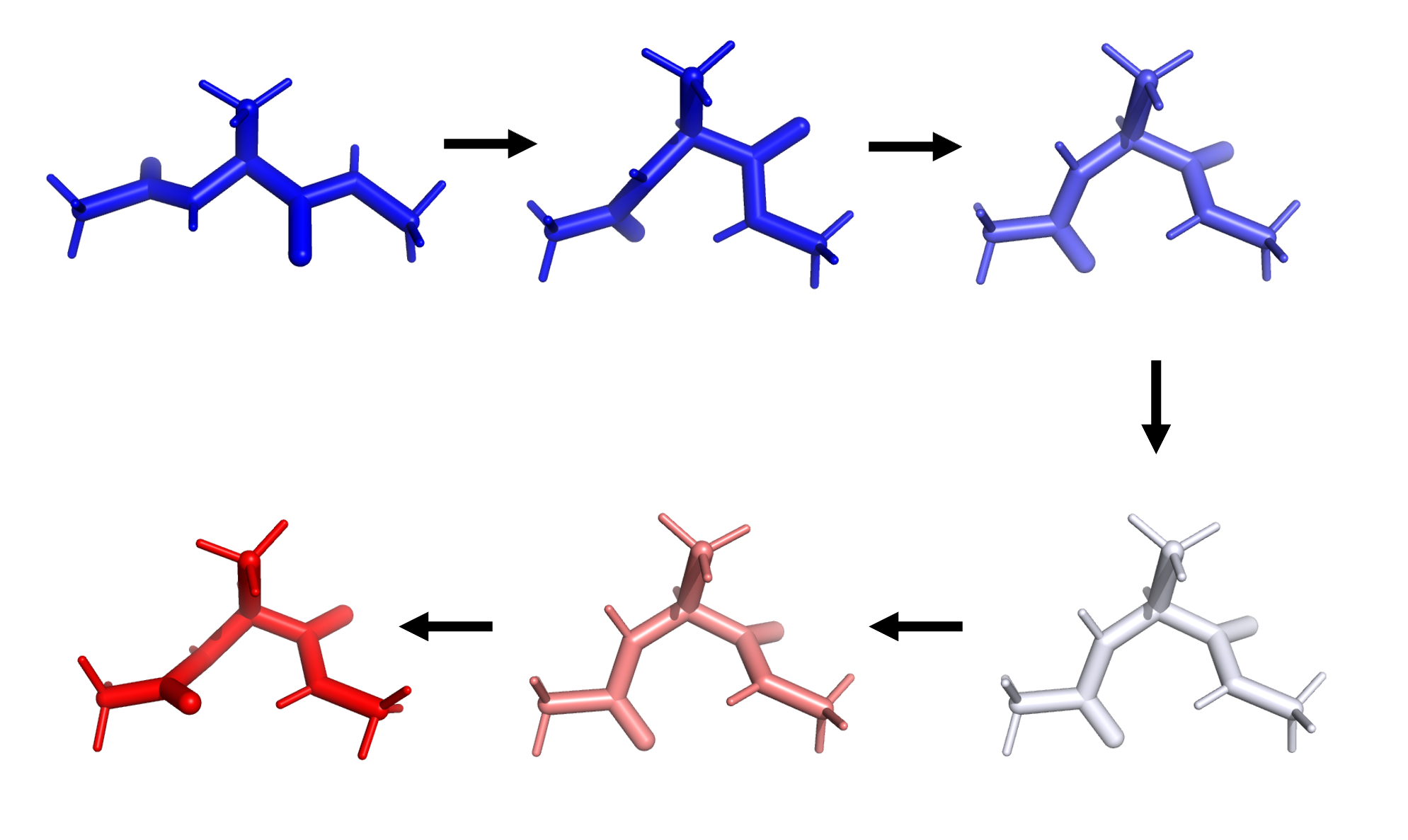}
    \caption{\textbf{ADP flow trajectory.} Snapshots of generated trajectory colored by the value of $h$, from dark blue at $A$ and red at $B$. The transition state is correctly identified as the white snapshot.}
\end{figure}

\begin{figure}[!h]
    \centering
	\includegraphics[width=\linewidth]{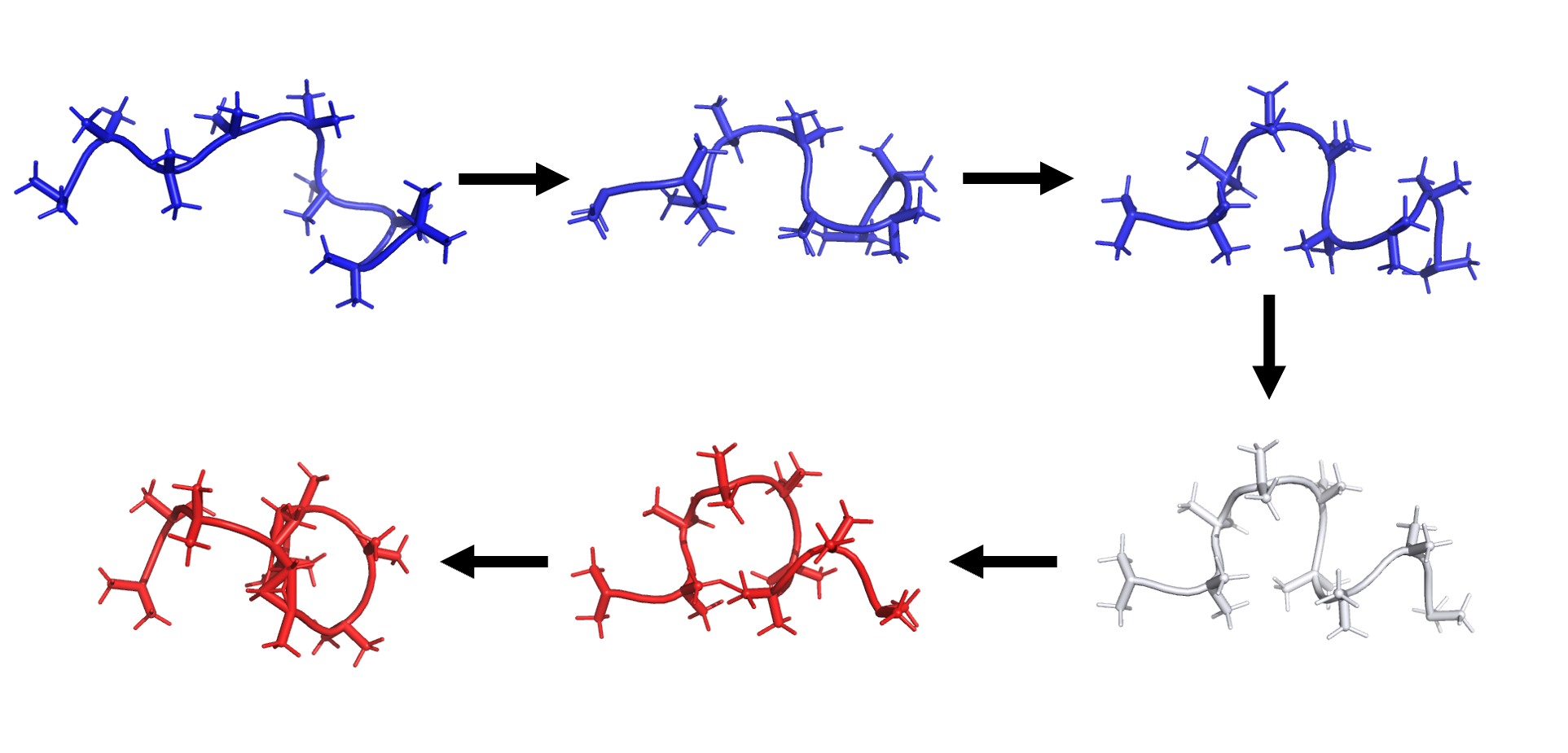}
    \caption{\textbf{AIB9 flow trajectory.} Snapshots of generated trajectory colored by the value of $\phi$. The start ($A$) and end ($B$) states can be visualized as dark blue and red and the transition state, identified correctly, is white.}
\end{figure}

\subsection{Rate Constant Estimation for the ADP $C_7^{\mathrm{eq}} \to C_7^{\mathrm{ax}}$ Transition}
\label{app:adp_rate}

\paragraph{Weighted ensemble setup.}

To assess the utility of the learned potential $h$ as a CV, we compute the rate constant for the $C_7^{\mathrm{eq}} \to C_7^{\mathrm{ax}}$ transition in alanine dipeptide (ADP) using the weighted ensemble (WE) method~\citep{huber1996weighted, zuckerman2017weighted}, as implemented in the WESTPA~2.0 package~\citep{russo2022westpa}, and compare against a CV constructed directly from the backbone dihedrals.

Weighted ensemble partitions the chosen CV space into bins and maintains a target number of trajectory replicas per bin via splitting and merging, while preserving unbiased trajectory weights. The flux of weight into the product state $C_7^{\mathrm{ax}}$ provides an unbiased estimator of the forward rate constant. To avoid hand-tuning bin boundaries for each CV, we use the minimal adaptive binning (MAB) scheme~\citep{torrillo2021minimal}, which dynamically places bins along the CV based on the leading and trailing replicas at each iteration and concentrates resolution near the progressing front of the ensemble. We initialize all replicas in $C_7^{\mathrm{eq}}$ and run WE until the running rate estimate stabilizes.

\paragraph{Collective variables.}
We compare two choices of CV for MAB binning:
\begin{itemize}
    \item \textbf{Backbone-dihedral CV.} Following standard practice for ADP, we define a scalar progress coordinate as the maximum angular displacement along the $(\phi, \psi)$ backbone dihedrals from the reactant minimum:
    \[
    \xi_\text{dih}(x) := \max\bigl(\, d_\phi(\phi(x), \phi_{\mathrm{ax}}),\; d_\psi(\psi(x), \psi_{\mathrm{ax}}) \,\bigr),
    \]
    where $d_\phi, d_\psi$ denote angular distances on the circle and $(\phi_{\mathrm{ax}}, \psi_{\mathrm{ax}})$ is the $C_7^{\mathrm{ax}}$ minimum. This CV is a natural baseline since $(\phi, \psi)$ are the canonical slow degrees of freedom for ADP.
    \item \textbf{Learned-potential CV.} We use $h_\theta(x)$ directly as a scalar CV, with MAB operating between the values of $h$ at the two endpoint states.
\end{itemize}

For both CVs, we use 40 MAB bins with a target of 10 walkers per bin and a resampling interval of 100~fs.

\begin{figure}[t]
    \centering
    \includegraphics[width=0.5\linewidth]{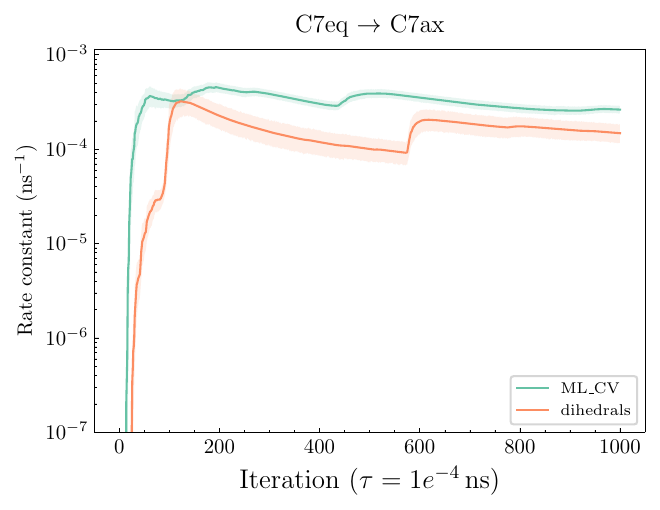}
    \caption{\textbf{Rate constant estimation for the $C_7^{\mathrm{eq}} \to C_7^{\mathrm{ax}}$ transition in ADP.} We compare WE rate estimates obtained using the learned potential $h$ as the collective variable against those obtained using the backbone-dihedral CV $\xi_\text{dih}$. Both runs use the WESTPA~2.0 implementation with minimal adaptive binning. Using $h$ yields earlier first-passage events, faster convergence of the running rate estimate, and tighter confidence intervals, indicating that $h$ provides a more informative collective variable for driving the transition.}
    \label{fig:ADP_rate}
\end{figure}

\paragraph{Results.}
Figure~\ref{fig:ADP_rate} compares the running WE rate estimates obtained with the two CVs. Using $h$ as the binning coordinate produces earlier first-passage events, faster convergence of the running rate estimate, and tighter confidence intervals than the backbone-dihedral CV. We attribute this improvement to two properties of $h$: it is a globally smooth, scalar-valued progress coordinate that aligns with the dominant transition direction, whereas the dihedral-based CV reduces a two-dimensional periodic space to a single scalar via a $\max$ operation that can cause replicas with different mechanistic character to fall into the same bin. Together, these results suggest that $h$ serves as a more informative CV for driving the $C_7^{\mathrm{eq}} \to C_7^{\mathrm{ax}}$ transition in WE simulations.

\subsection{Dataset size ablation}

The difference in performance on training on 100K vs 10K trajectories is reported in the table below:

\begin{table}[tbh]
	\caption{\textbf{Quantitative results across systems.} Values are computed over 256 trajectories. In both systems, the generated trajectories closely reproduce the values observed in the reactive trajectory ensemble.}
	\label{tab:datablation_results}
	\centering
	\begin{tabular}{lcccccc}
		\toprule
		System & Method & No. of trajectories &Features & Completion rate & $\mathbb{T}$-$W_2$\\
		\midrule
		ADP & TPS & 100K &$\mathbb{R}^{3N}$ & 1 & $0.0875 \pm 0.006$\\
         ADP & FM & 10K&$\mathbb{R}^{3N}$ & 0.6250 & $0.3683 \pm 0.009$\\
        ADP & FM & 100K&$\mathbb{R}^{3N}$ & 0.7734 & $0.6453 \pm 0.016$\\
       ADP & FM & 10K&Dihedrals & 0.9531 & $0.8573 \pm 0.005$\\
        ADP & FM & 100K&Dihedrals & 0.9804 & $0.7220 \pm 0.017$\\
        
        \midrule
        AIB9 & TPS & 100K&$\mathbb{R}^{3N}$ & 1 & $0.1011 \pm 0.034$ \\
        AIB9 & FM & 10K&Backbone Dihedrals & 0.7343 & $1.096 \pm 0.014$\\
        AIB9 & FM & 100K&Backbone Dihedrals & 0.7578 & $0.9137 \pm 0.007$\\
        
		\bottomrule
	\end{tabular}
\end{table}



\end{document}